\documentclass[twoside]{article}
\pdfoutput=1

\usepackage[accepted]{aistats2022}

\usepackage{graphicx}
\usepackage{hyperref}
\usepackage{booktabs} % for professional tables
\usepackage{amsmath}
\usepackage{amssymb}
\usepackage{cite}
\usepackage{easyeqn}

\usepackage{todonotes}
\usepackage{array,multirow}
\usepackage{makecell}
\usepackage{amsmath, amssymb, amsthm, bbm}
\usepackage{mathrsfs}

\usepackage{caption}
\usepackage{subcaption}

\graphicspath{{../../figs/}}

\usepackage{tikz,tikz-cd}
\newtheorem{ther}{Theorem}
\newtheorem{prop}[ther]{Proposition}

\begin{document}

\runningauthor{Velikanov, Kail, Anokhin, Vashurin, Panov, Zaytsev, Yarotsky}

\twocolumn[

\aistatstitle{Embedded Ensembles: Infinite Width Limit and Operating Regimes}

\aistatsauthor{Maksim Velikanov  
\And
Roman Kail \And
Ivan Anokhin
\And
Roman Vashurin}

\aistatsaddress{Skoltech \And Skoltech \And 
Skoltech \& Yandex \And
Skoltech} 

\aistatsauthor{
Maxim Panov \And 
Alexey Zaytsev \And 
Dmitry Yarotsky}

\aistatsaddress{Skoltech \And Skoltech \And 
Skoltech} ]

\begin{abstract}
  A memory efficient approach to ensembling neural networks is to share most weights among the ensembled models by means of a single reference network. We refer to this strategy as \emph{Embedded Ensembling} (EE); its particular examples are BatchEnsembles and Monte-Carlo dropout ensembles. In this paper we perform a systematic theoretical and empirical analysis of embedded ensembles with different number of models. 
  Theoretically, we use a Neural-Tangent-Kernel-based approach to derive the wide network limit of the gradient descent dynamics. In this limit, we identify two ensemble regimes -- \emph{independent} and \emph{collective} -- depending on the architecture and initialization strategy of ensemble models. We prove that in the independent regime the embedded ensemble behaves as an ensemble of independent models.  We confirm our theoretical prediction with a wide range of experiments with finite networks, and further study empirically various effects such as transition between the two regimes, scaling of ensemble performance with the network width and number of models, and dependence of performance on a number of architecture and hyperparameter choices. 

\end{abstract}

\section{Introduction}
\label{sec:introduction}

A common strategy of improving accuracy of predictive models is \emph{model ensembling}~\cite{dietterich2000ensemble}. 
In its simplest form, several models are constructed independently, and their outputs are averaged. 
Despite its simplicity, this strategy is very reliable and efficient, almost invariably improving the accuracy and robustness of predictions~\cite{dusenberry2020efficient}.

However, a major downside of this strategy is a substantial increase of resources required for model training and execution: training time, inference time, and required storage scale linearly with the number of  models in the ensemble.
This downside is especially acute for deep neural networks (DNNs) since they are already complex -- as a result, DNN ensembles become challenging or even infeasible in many applications~\cite{schwenk2000boosting,huang2017snapshot}. 

Recently significant attention was paid to the construction of ``lightweight'' ensembles that mitigate this issue~\cite{wen2019batchensemble,havasi2020training,rame2021mixmo,wenzel2020hyperparameter,zhang2021ex}. 
A~lightweight ensemble attempts to retain the accuracy gain from ensembling while relaxing requirements for a particular resource. 
For example, snapshot ensembles~\cite{huang2017snapshot,garipov2018loss} reduce the ensemble training time (without significantly affecting the storage and inference time). Lightweight ensembles typically have a lower accuracy than the standard independent ensembles of the same size, because of a lower diversity of their members. 

In this work we address what we call \emph{Embedded Ensembles} (EE). Their common idea is to construct different models by some kind of perturbation of a single reference neural network. Examples of EE include Monte-Carlo (MC) dropout ensembles~\cite{gal2016dropout} and BatchEnsembles~\cite{wen2019batchensemble}. Most weights in an embedded ensemble are just the shared reference network weights, so this ensemble requires much less storage than a respective ensemble of independent reference networks. Furthermore, if the perturbation is restricted to the last layers, then network computations can be efficiently reused among ensembled models
making computation time comparable to that of a single model.

The price one pays for this efficiency is the lower accuracy of embedded ensembles. In fact, while the accuracy of the usual independent ensembles only increases with additional models, it was observed empirically~\cite{havasi2020training} that the accuracy of embedded ensembles can degrade when the number of ensemble members is large.   

The primary purpose of this work is to systematically explore how performance of embedded ensembles scales with the number of models. An additional important factor that we consider is the size of the reference network. Intuitively, larger reference networks can accommodate more uncorrelated models and so provide higher ensemble accuracy. We confirm this intuition, both theoretically and empirically. 

\paragraph{Our contribution.}
We perform an extensive theoretical and empirical study of Embedded Ensembles. 

\begin{itemize}
    \item We describe the behaviour of Embedded Ensembles in the limit of infinite reference network width. Particularly, we derive dynamic equation of EE model outputs describing their evolution under gradient descent. Also, we characterize at initialization the distribution of ensemble outputs and Neural Tangent Kernel of the ensemble.

    \item In the infinite width limit we identify \textit{independent} and \textit{collective} operating regimes of Embedded Ensembles. In the independent regime we show that EE is fully identical to the ensemble of independent reference networks. Also, we propose to use different gradient scalings for independent and collective regimes to ensure proper behavior of EEs with large number of ensemble models. We show that the operating regime of Embedded Ensemble is determined by the structure of individual parameters and their initialization strategy.

    \item We perform extensive experiments with embedded ensembles on the CIFAR100 data set. We empirically observe the  collective and independent regimes and demonstrate the transition between them.  We observe that finite-width EEs in the independent regime have an optimal number of models at which the highest accuracy is achieved; in agreement with our theory this optimal number increases with the network width. We further explore, both empirically and theoretically, a number of architecture modifications and the scaling of the learning rates. 
\end{itemize}

\begin{figure}[t!]
  \centering
    
  \begin{tikzpicture}[scale=0.8]
    \def\r{0.08}
    \def\l{2}
    \def\ll{0.8}
    \def\d{0.3}
    \def\z{0.3}
    \def\q{0.05}
    \def\nlayers{3}
    \def\nchannels{4}

    \def\colorPallete{{"0 1 0 0", "0 0 1 0", "0 0 0 1"}}

    \foreach \layer in {1,...,\nlayers}{
    \foreach \channel in {1,...,\nchannels}{

    \coordinate (post) at (\layer*\l, \channel);
    \fill  (post) circle (\r);
    \foreach \chan in {1,...,\nchannels}{
      \coordinate (pre) at (\layer*\l+\ll, \chan);
      \fill  (pre) circle (\r);
      \draw (post) -- (pre);
    }

    \node[circle,draw,minimum size=4mm, inner sep=0pt, outer sep=0pt] (phi) at (\layer*\l+0.5*\l+0.5*\ll,\channel) {$\phi$};

    \node[circle,draw,fill=black, minimum size=\r, inner sep=0pt, outer sep=2.5] (pre) at (\layer*\l+\ll, \channel) {};

    \node[circle,draw,fill=black, minimum size=\r, inner sep=0pt, outer sep=2.5] (post) at (\layer*\l+\l, \channel) {};

    \foreach \model in {-1,0,1}{
      \coordinate (modelPre) at (\layer*\l+0.25*\l+0.75*\ll-\q, \channel+\z*\model);
      \coordinate (modelPost) at (\layer*\l+0.25*\ll+0.75*\l+\q, \channel+\z*\model);
      
    \pgfmathparse{\colorPallete[\model+1]};
    \definecolor{currentColor}{rgb}{\pgfmathresult};
      
      \fill [currentColor] (modelPre) circle (\r);
      \draw [currentColor] (modelPre) -- (pre);
      \draw [currentColor] (modelPre) -- (phi);
      
      \fill [currentColor] (modelPost) circle (\r);
      \draw [currentColor] (modelPost) -- (post);
      \draw [currentColor] (modelPost) -- (phi);
    }
    }
    }

    %%%%%%%%%%%%%%%%%%% legend

    \def\posx{-1}
    \def\posy{3}

    \foreach \model in {1,2,3}{
      
    \pgfmathparse{\colorPallete[\model-1]};
    \definecolor{currentColor}{rgb}{\pgfmathresult};

    \fill [currentColor] (\posx, {\posy-0.5*(\model-2)}) circle (\r);
    \node [] at (\posx+1, {\posy-0.5*(\model-2)}) {model \model};
    }

    \fill [] (\posx, {\posy-0.5*2}) circle (\r);
    \node [] at (\posx+1, {\posy-0.5*2}) {common};
  \end{tikzpicture}

    \caption{All models in the BatchEnsemble have common fully-connected (or convolutional) weights (colored  black), and a small number of pre- and post- activation modulations (colored red, green or blue) that differ for each model. For each model, only the respective family of modulations (i.e., red, green or blue) is active on a forward pass. }
    \label{fig:teaser}
  \end{figure}
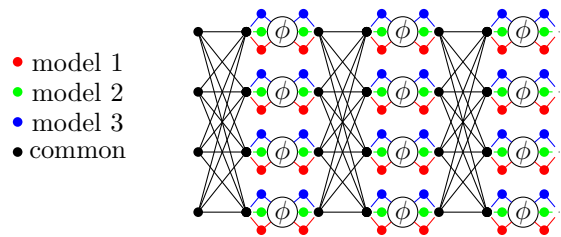 
  
\section{Related Work}
  Early attempts to reach a compromise between increased performance of ensembles and low computational and memory costs of single models were focused on so-called Monte-Carlo Dropout in the context of uncertainty estimation~\cite{gal2016,gal2016dropout}. Recently, there was an increased interest in improving predictive characteristics of embedded ensembles. Several different approaches were suggested, such as BatchEnsembles~\cite{wen2019batchensemble}, MIMO~\cite{havasi2020training}, Masksembles~\cite{durasov2021masksembles} and Orthogonal Dropout~\cite{zhang2021ex}.

  There is very limited literature analysing embedded ensembles' characteristics as a function of parent network size and number of ensemble models. \cite{seoh2020qualitative} focuses on MC Dropout, but mostly looks at uncertainty estimation. \cite{Sicking2020DropoutCharacteristics} explores connection of MC Dropout with Gaussian processes. \cite{Verdoja2021Dropout} focuses on MC Dropout, taking a look on how well it captures variance in the data.

  One example of the study more relevant to ours is~\cite{zhang2021ex}, which introduces a special variation of MC Dropout, and provides insights in how well it performs for different ensemble sizes.

  Fully-connected neural networks in the limit of infinite width are equivalent to Gaussian Processes (GP) at initialization~\cite{lee2018deep} and behave as linear models under gradient descent~\cite{lee2019wide} with Neural Tangent Kernel (NTK; \cite{jacot2018neural}) as a matrix governing linear training dynamics. Although actual NNs are not fully described by their infinite width limit~\cite{Lee2020FiniteVI,Hanin2020FiniteDA,Liu2020OnTL,Huang2020DynamicsOD,Lewkowycz2020TheLL}, this limit becomes a convenient way to study NNs and was generalized to a number of modern architectures:  RNN~\cite{Alemohammad2021TheRN}, CNN~\cite{Arora2019OnEC}, Attention~\cite{Hron2020InfiniteAN}, Deep Bayesian Ensembles~\cite{He2020BayesianDE}. The latter work uses GPs and NTK to analytically study ensembles of independent DNNs, while in current work we use these tools to study Embedded Ensembles.

\section{Embedded Ensembles}\label{sec:EE_description}
  Consider a \textit{reference network}\footnote{The output $f$ of a predictive model generally depends on the input $\mathbf x$ and the model parameters that in turn depend on the training time $t$. For brevity, we will occasionally omit some arguments of $f$ in formulas.} $f(\mathbf{w}, \mathbf{u}, \mathbf{x})$ with two sets of parameters $\mathbf{w}, \mathbf{u}$, and inputs $\mathbf{x}$. To build an Embedded Ensemble of size $M$ based on this reference model we choose a single set of parameters $\mathbf{w}$ to be shared across all $M$ ensemble models, and $M$ sets $\{\mathbf{u}_{\alpha}\}_{\alpha=1}^M$ of parameters $\mathbf{u}$ to serve as individual parameters of ensemble models. Then the ensemble prediction at input $\mathbf{x}$ can be written as
  \begin{equation*}
    f^{\text{ens}}(\mathbf{x}) = \frac{1}{M} \sum_{\alpha = 1}^M f_\alpha(\mathbf{x}), \quad f_\alpha(\mathbf{x}) \equiv f(\mathbf{w}, \mathbf{u}_\alpha, \mathbf{x}).
  \end{equation*}
  The sharing of the weights $\mathbf{w}$ across ensemble models makes EE memory efficient. Particularly, in the case of the reference model having much fewer parameters $\mathbf{u}$ than $\mathbf{w}$, the total number of parameters in EE is approximately the same as in the reference model. On the other hand, the sharing of the common parameters makes ensemble models $f_\alpha(\mathbf{x})$ correlated and interacting during training. The study of these interaction effects, and situations where they are suppressed, is the goal of the current work. 

  We focus on two particular examples of embedded ensembles: BatchEnsembles and MC dropout ensembles.

\paragraph{BatchEnsembles.}
  In a BatchEnsemble~\cite{wen2019batchensemble,wen2020improving,dusenberry2020efficient}, each model  has the same architecture as the reference model up to an additional rescaling of the inputs and outputs of neurons; we will call these rescalings \emph{modulations}.
  
  A standard fully-connected layer in a neural network can be written as 
  \begin{EQA}[c]
    z^{l}_{j} = \frac{1}{\sqrt{N_{l-1}}} \sum\limits_{i = 1}^{N_{l - 1}} x_{i}^{l-1} W^{l}_{ij} + b^l_j,
    \quad
    x^{l}_{j} = \phi\bigl(z^l_{j}\bigr),
  \label{eq:linear_layer}
  \end{EQA}
  where index $l$ corresponds to a layer, $N_l$ denotes the number of neurons in $l$-th layer, $z^{l}_{j}$ are usually called pre-activations and $x^{l}_{j}$ denotes the output of the $j$-th neuron in the $l$-th layer. The factor $N_{l-1}^{-1/2}$ is written for future convenience of passing to the wide network limit.  In the BatchEnsemble, in each neuron of each submodel $\alpha = 1, \ldots, M$, the activation $x = \phi(z)$ is replaced by the ``modulated'' activation:
  \begin{equation}
  \label{eq:batch_layer}
    \begin{cases}
      z^{l}_{\alpha j} = \frac{1}{\sqrt{N_{l-1}}} \sum\limits_{i = 1}^{N_{l - 1}} x_{\alpha i}^{l-1} W^{l}_{ij} + b^l_j,
      \\
      x^{l}_{\alpha j} = u^{l}_{\alpha j} ~ \phi\bigl(v^{l}_{\alpha j} z^l_{\alpha j}\bigr),
    \end{cases}
  \end{equation}
where $\mathbf{u}_{\alpha} = \{u^l_{\alpha j}, v^l_{\alpha j}\}$ are the \emph{modulating weights} depending in particular on the model $\alpha$ in the ensemble, the layer $l$ in the reference network and an index $j$ of a neuron in a fully-connected layer. At the same time, all models of the ensemble share the same set of weights $\mathbf{w} = \{W^{l}_{ij}, b^l_j\}$ inherited from the reference network. The full scheme of BatchEnsemble architecture is presented on Figure~\ref{fig:teaser}. 
  
There is a natural generalization of this structure to convolutional layers with multiple channels:
the modulating weights depend on the channel $j$ and are shared by all neurons of the channel.

\paragraph{MC dropout ensembles.}
Dropout was originally proposed as a regularization technique for deep learning used at training time~\cite{JMLR:v15:srivastava14a}. However, when applied at inference time, it can also be viewed as an ensembling approach. 

It can be defined as a modification of BatchEnsemble in which only the post-activation modulating weights $u_{\alpha j}^l$ are present in equation~\eqref{eq:batch_layer} (no $v_{\alpha j}^l$), and these modulating weights are not optimized after random initialization (i.e., only the shared weights $\mathbf w$ are learned).
  
When the modulating weights $u_{\alpha j}^l$ have a Bernoulli 0--1 distribution, their realizations can be referred to as ``dropout masks'' because they effectively eliminate some neurons from the network. We will, however, allow $u_{\alpha j}^l$ to have general distributions.
  
A particular special case of MC dropout ensembles occurs if the modulating weights $u_{\alpha j}^l$ are present only in the last hidden layer. In this case all the network computations up to the last layer are shared among the ensemble models, allowing to substantially accelerate the overall computation of the ensemble. We will refer to this case as the \emph{Last Layer Dropout (LLD) ensemble.}

\subsection{Training of Embedded Ensembles}\label{sec:train_emb_ens}
  We assume that models in EE are trained simultaneously with a version of (Stochastic) Gradient Descent. At each GD step the batch of training pairs $\mathcal{B}=\{(\mathbf{x}_b, y_b)\}_{b = 1}^B$ is given to each ensemble model $f_\alpha(x)$ and corresponding empirical loss is calculated
  \begin{equation}
    \mathcal{L}_\alpha(\mathbf{w}, \mathbf{u}_\alpha) = \frac{1}{B}\sum\limits_{b = 1}^B L(f(\mathbf{w}, \mathbf{u}_\alpha,\mathbf{x}_b), y_b).
  \end{equation}
  Here $L(\hat{y}, y)$ is a single point loss function. The parameters of the ensemble are then updated based on loss gradients of all ensemble models 
  \begin{align}
    \label{eq:individual_grads}
    \Delta \mathbf{u}_\alpha &= -\eta_{u} \frac{\partial \mathcal{L}_\alpha(\mathbf{w}, \mathbf{u}_\alpha)}{\partial \mathbf{u}_\alpha}, \\
    \label{eq:common_grads}
    \Delta \mathbf{w} &= -\eta_{w} \frac{\gamma(M)}{M} \sum\limits_{\alpha=1}^M \frac{\partial \mathcal{L}_\alpha(\mathbf{w}, \mathbf{u}_\alpha)}{\partial \mathbf{w}}.
  \end{align}
  Here different learning rates $\eta_w, \eta_u$ are used emphasize that common and individual parameters have different roles in the EE and thus it might be beneficial to learn them with different rates. Note that according to~\eqref{eq:common_grads} each ensemble model contributes to update of common parameters $\Delta \mathbf{w}$. A natural question then is how to correctly accumulate gradients from different models $\alpha$, which becomes especially important when the number of models $M$ becomes big. We propose to control accumulation of gradients by means of scaling factor $\gamma(M)$. In the sequel we will argue that the natural choice for scaling factor is either $\gamma(M)=1$ or $\gamma(M)=M$ depending on whether the ensemble is in independent or collective regime.    

\section{Theoretical Analysis}
\label{sec:theory}
  One particular goal of EE's analysis is to quantify and minimize the dependencies between different ensemble models in order to make performance of EE close to performance of independent ensemble. We find that this analysis simplifies significantly in the limit of infinite network width, where ensemble models can be made completely independent.  

  Let us first informally discuss in which sense different models in EE could be independent despite sharing most of their parameters. We argue that independence of two different models can be achieved if they are independent at initialization and have \textit{dynamic independence}. At initialization, the outputs $f_\alpha(\mathbf{x})$ and $f_\beta(\mathbf{x}')$ of two different EE models at any pair of inputs  are random variables w.r.t. random initialization of parameters $\mathbf{w}, \mathbf{u}_\alpha, \mathbf{u}_\beta$ and thus their independence can be defined in the standard sense. 

  To define {dynamic independence} consider a single gradient descent step where the update of common parameters~\eqref{eq:common_grads} is composed of contributions of each ensemble model: $\Delta \mathbf{w} = \sum_\beta \Delta_\beta \mathbf{w}$. Then, the change of model $\alpha$ output $\Delta f_\alpha(\mathbf{x})$ induced by update $\Delta \mathbf{w}$ is also decomposed into contributions $\Delta_\beta \mathbf{w}$ from different EE models:
  \begin{equation}
    \hspace{-5pt} \Delta f_\alpha(\mathbf{x}) = \sum\limits_{\beta=1}^M\frac{\partial f_\alpha(\mathbf{x})}{\partial \mathbf{w}} \Delta_\beta \mathbf{w} \propto -\sum\limits_{\beta=1}^M \frac{\partial f_\alpha(\mathbf{x})}{\partial \mathbf{w}} \frac{\partial \mathcal{L}_\beta}{\partial \mathbf{w}}.
  \end{equation}
  We see now that \emph{dynamic independence} of model $\alpha$ from model $\beta$ can be defined as zero contribution of updates $\Delta_\beta \mathbf{w}$ on output $f_\alpha(\mathbf{x})$. This is achieved when model $\alpha$ output gradient is orthogonal to model $\beta$ loss gradient: $\frac{\partial f_\alpha(\mathbf{x})}{\partial \mathbf{w}} \frac{\partial \mathcal{L}_\beta}{\partial \mathbf{w}}=0$. 

  We now turn to description of EE output dynamics and its connection to ensemble models independence.

\subsection{Output dynamics}\label{sec:output_dyn}

\paragraph{Single network.} 
  In the case of a single neural network $f(\mathbf{w},\mathbf{x})$, the output training dynamic can be conveniently described in terms of \textit{Neural Tangent Kernel} (NTK; \cite{jacot2018neural}). For the continuous time gradient descent (gradient flow) one has
  \begin{align}
    \label{eq:single_network_output_dynamic}
    &\frac{df(\mathbf{x})}{dt}=-\frac{1}{B}\sum\limits_{b = 1}^B \Theta(\mathbf{x}, \mathbf{x}_b, \mathbf{w})\frac{\partial L(f(\mathbf{x}_b), y_b)}{\partial f(\mathbf{x}_b)}, \\
    &\Theta(\mathbf{x},\mathbf{x}',\mathbf{w}) = \frac{\partial f(\mathbf{w},\mathbf{x})}{\partial \mathbf{w}}\frac{\partial f(\mathbf{w},\mathbf{x}')}{\partial \mathbf{w}}.
  \end{align}
  Here $\Theta(\mathbf{x},\mathbf{x}',\mathbf{w})$ is the Neural Tangent Kernel of network $f(\mathbf{x},\mathbf{w})$. Note that the dependence of the NTK on parameters $\mathbf{w}$ in~\eqref{eq:single_network_output_dynamic} does not allow to completely ignore parameters $\mathbf{w}$ and obtain a closed dynamic equation in terms of only outputs $f(\mathbf{x})$. However, the situation greatly simplifies in the limit of the infinite network width $N\rightarrow \infty$, where the following three statements hold~\cite{jacot2018neural,lee2019wide}:
  \begin{itemize}
    \item[(A1)] The network output $f(\mathbf{x})$ is a draw from a Gaussian Process (GP) at initialization.
    \item[(A2)] The NTK $\Theta(\mathbf{x}, \mathbf{x}')$ converges to a non-random deterministic value at initialization. %\todo{$\Theta(\mathbf{x}, \mathbf{x}')$ is not defined.}
    \item[(A3)] The NTK $\Theta(\mathbf{x}, \mathbf{x}')$ stays constant during training.
  \end{itemize}
  In particular, the combination of statements (A2, A3) means that~\eqref{eq:single_network_output_dynamic} is a closed-form dynamic equation in network outputs. 

\paragraph{Embedded Ensemble.} 
  Now we extend the approach outlined above and obtain dynamic equation of EE outputs $f_\alpha(\mathbf{x})$ as well as analogues of statements (A1-A3). Dynamic equation can be obtained from parameter update rules~\eqref{eq:individual_grads}-\eqref{eq:common_grads} and $\frac{\partial \mathcal{L}}{\partial \mathbf{w}}=\frac{\partial \mathcal{L}}{\partial f}\frac{\partial f}{\partial \mathbf{w}}$. Detailed derivation as well as background on NTK can be found in Sections~\ref{sec:NTK_background} and~\ref{sec:EEs_NTK}. The result is    
  \begin{equation}\label{eq:EE_outputs_dynamic}
    \frac{d f_\alpha(\mathbf{x})}{dt} = -\frac{1}{B}\sum\limits_{b,\beta}\Theta_{\alpha\beta}(\mathbf{x},\mathbf{x}_b) \frac{\partial L(f_\beta(\mathbf{x}_b), y_b)}{\partial f_\beta(\mathbf{x}_b)}.
   \end{equation}
  Here NTK $\Theta_{\alpha\beta}(\mathbf{x},\mathbf{x}')$ acquires ensemble model indices $\alpha,\beta$  and can be represented as the sum of contributions from common and individual parameters $\mathbf{w}$ and $\mathbf{u}_\alpha$:
  \begin{equation}\label{eq:EE_NTK0}
    \begin{split}
      &\Theta_{\alpha\beta}(\mathbf{x},\mathbf{x}') = \frac{\gamma(M)}{M}\Theta^{\text{com}}_{\alpha\beta}(\mathbf{x},\mathbf{x}') + \delta_{\alpha\beta}\Theta_\alpha^{\text{ind}}(\mathbf{x},\mathbf{x}'),\\
      &\Theta^{\text{com}}_{\alpha\beta}(\mathbf{x},\mathbf{x}') = \frac{\partial f(\mathbf{w}, \mathbf{u}_\alpha,\mathbf{x})}{\partial \mathbf{w}} \frac{\partial f(\mathbf{w}, \mathbf{u}_\beta,\mathbf{x}')}{\partial \mathbf{w}},\\
      &\Theta^{\text{ind}}_\alpha(\mathbf{x},\mathbf{\mathbf{x}}') = \frac{\partial f(\mathbf{w}, \mathbf{u}_\alpha,\mathbf{x})}{\partial \widetilde{\mathbf{u}}_\alpha} \frac{\partial f(\mathbf{w}, \mathbf{u}_\alpha,\mathbf{x}')}{\partial \widetilde{\mathbf{u}}_\alpha}.
    \end{split}
  \end{equation}
  In the definition of $\Theta^{\text{ind}}$ above, $\widetilde{\mathbf{u}}_\alpha$ denotes the subset of trainable individual parameters. In particular, it means that for MC dropout ensembles $\Theta^{\text{ind}}_\alpha(\mathbf{x},\mathbf{x}')=0$ since all individual weights $\mathbf{u}_\alpha$ are nontrainable.  

  We see from~\eqref{eq:EE_outputs_dynamic} that if the NTK vanishes between different ensemble models $\alpha\ne\beta$ (i.e. $\Theta_{\alpha\beta}(\mathbf{x},\mathbf{x}')\equiv0$), then the model $\beta$ does not contribute to the update of the model $\alpha$, i.e. we have dynamic independence.  

  We formulate now general results on the wide network limit and model independence for BatchEnsembles and MC dropout ensembles. We consider architecture given by~\eqref{eq:batch_layer} with $L$ hidden layers and common parameters at layer $W^l_{ij}$ initialised as i.i.d. Gaussians $\mathcal{N}(0, 1)$. In these architectures we allow parameters $u^l_\alpha$, $v^l_\alpha$ to be absent at some layers of the network, and only require them to be i.i.d. random variables which do not vanish identically and with finite variances. Then we take the limit of infinite reference network layer widths $N_l\rightarrow \infty$ sequentially starting from the layer $l=1$. In this setting we describe the network outputs $f_\alpha(\mathbf{x})$ and ensemble NTK~\eqref{eq:EE_NTK0} at initialization with ensemble analogues of statements (A1, A2). After that we assert that the vanishing mean of the last layer modulations $U_1^L\equiv\mathbb{E}[u^L_{\alpha j}]$ or $U_1^L\equiv\mathbb{E}[t^L_{\alpha j}]$ is a sufficient condition of ensemble model independence. Moreover, to a certain extent, this condition is also necessary if the activation belongs to the family $\mathcal{S}$ of non-negative non-decreasing locally Lipschitz functions (which include ReLU and the standard sigmoid). Our main results are collected in the following two theorems.

  \begin{ther}[\textbf{Outputs $f_\alpha$}]
  \label{theorem_1}\hfill
  
  \vspace{-4mm}
  \begin{enumerate}
  \setlength\itemsep{-0.5mm}
    \item \textbf{(Gaussianity)} Consider a BatchEnsemble or MC dropout ensemble at initialization. Then in the sequential infinite-width limit $N_l\rightarrow\infty$ the collection of ensemble model outputs $f_\alpha(\mathbf{x})$ converge (in law) to a zero mean Gaussian Process (GP). 
    \item \textbf{(Independence)} If $U_1^L=0$, then the GP covariance $\mathbb{E} [f_\alpha(\mathbf{x})f_\beta(\mathbf{x}')]=0$ for all $\mathbf x,\mathbf x'$ and different ensemble models $\alpha\ne\beta$. 
    \item \textbf{(Breakdown of independence)}  Let the activation function $\phi \in \mathcal{S}$ and $U_1^L\ne 0$. Then $\mathbb{E} [f_\alpha(\mathbf{x})f_\beta(\mathbf{x}')]>0$ for all $\alpha\ne\beta$ and all (resp., all linearly independent) pairs of non-zero inputs $\mathbf{x},\mathbf{x}'$ for network depths $L>1$ (resp., $L=1$). 
    \end{enumerate}
  \end{ther}

  \pagebreak

  \begin{ther}[\textbf{NTK $\Theta_{\alpha\beta}$}]\hfill
  \label{theorem_2}
  \vspace{-4mm}
  \begin{enumerate}
  \setlength\itemsep{-0.5mm}
    \item \textbf{(Determinacy)} Consider a BatchEnsemble or MC dropout ensemble at initialization. Then in the sequential limit $N_l\rightarrow\infty$ the ensemble NTK $\Theta_{\alpha\beta}(\mathbf{x},\mathbf{x}')$ converges to a deterministic value. 
    \item \textbf{(Dynamic independence)} If $U_1^L=0$, then the ensemble NTK $\Theta_{\alpha\beta}(\mathbf{x},\mathbf{x}') = 0$ for all $\mathbf x, \mathbf x'$ and different ensemble models $\alpha\ne\beta$. 
    \item \textbf{(Breakdown of dynamic independence)} If $U_1^L\ne0$ and the activation function $\phi \in \mathcal{S}$, then $\Theta_{\alpha\beta}(\mathbf{x},\mathbf{x}')>0$ for all $\alpha\ne \beta$ and all pairs of inputs $\mathbf{x},\mathbf{x}'$ with  scalar product $\mathbf{x}^T\mathbf{x}'>0$. 
\end{enumerate}
  \end{ther}
  
  Note that we establish the breakdown of dynamic independence only in the sense that  $\Theta_{\alpha\beta}(\mathbf{x},\mathbf{x}')\not\equiv 0$ for $\alpha\ne \beta$ (not excluding in principle that the terms of the sum
  in equation~\eqref{eq:EE_outputs_dynamic} cancel out so that the orthogonality $\frac{\partial f_\alpha(\mathbf{x})}{\partial \mathbf{w}} \frac{\partial \mathcal{L}_\beta}{\partial \mathbf{w}}=0$ still holds).

  The proofs of the theorems are given in Section~\ref{sec:thm_proofs}. The key element of the proofs are recursive formulas for GP covariances and NTK's in different layers of the network (see Propositions~\ref{prop:1} and~\ref{prop:2} in Section~\ref{sec:SM_detailed_results}). These recursive formulas reveal a special role of post-activation modulations $u_\alpha^l$, and in particular the last layer modulations $u_\alpha^L$, for model independence. Heuristically, the role of the last layer $L$ is crucial because even if the model independence is present at earlier layers, it may get broken as the signal propagates to the network output.   
  The characterization of model independence provided by Theorems~\ref{theorem_1} and~\ref{theorem_2} only concerns the initial ($t=0$) states of the network. We additionally need the constancy of the NTK during training for the full result: in this case the dynamic equations~\eqref{eq:EE_outputs_dynamic} on $f_\alpha(\mathbf{x}, t)$  become closed and model independent, which ensures the independence of the trained model outputs $f_\alpha(\mathbf{x}, t>0)$ given their independence at $t=0$. In Section~\ref{sec:ntk_stat_proof} we show the constancy of the NTK during training using the approach from~\cite{Dyer2020AsymptoticsOW}.

  \begin{figure*}[tb]
    \centering
    \setlength{\fboxsep}{0pt}
    {\includegraphics[scale=0.55, trim=4mm 40mm 0mm 4mm, clip]{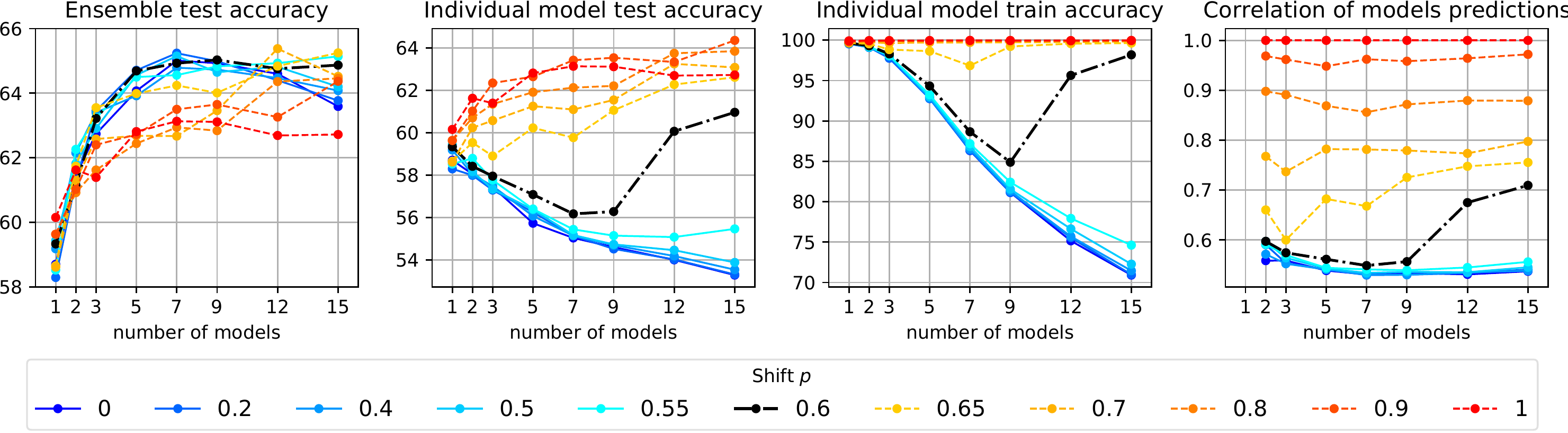}}
    \caption{BatchEnsemble with all modulations having the shifted initialization $\mathcal{N}(p, 1-p^2)$. \textbf{From left to right:} Accuracy of ensemble predictions on test set, mean accuracy of individual EE models on test set, mean accuracy of individual EE models on train set, mean binary correlation of class predictions of individual EE models on test set. We observe a transition point at $p\approx 0.6$, with ensembles below/above this point being in the independent/collective regime.
    }
    \label{fig:shifting_modulations}
  \end{figure*}

  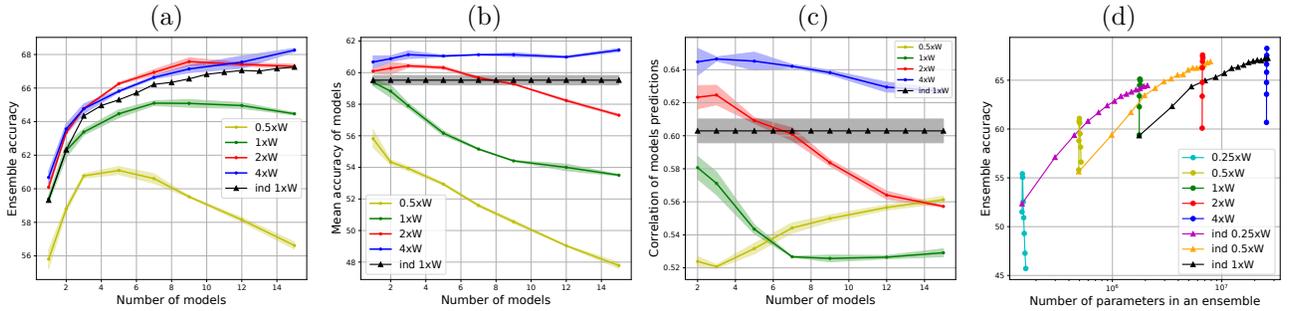
\begin{figure*}[tb]
    \centering
    \begin{picture}(487,115)
    \put(0,0){\includegraphics[scale=0.28, trim=80mm 2mm 75mm 3mm, clip]{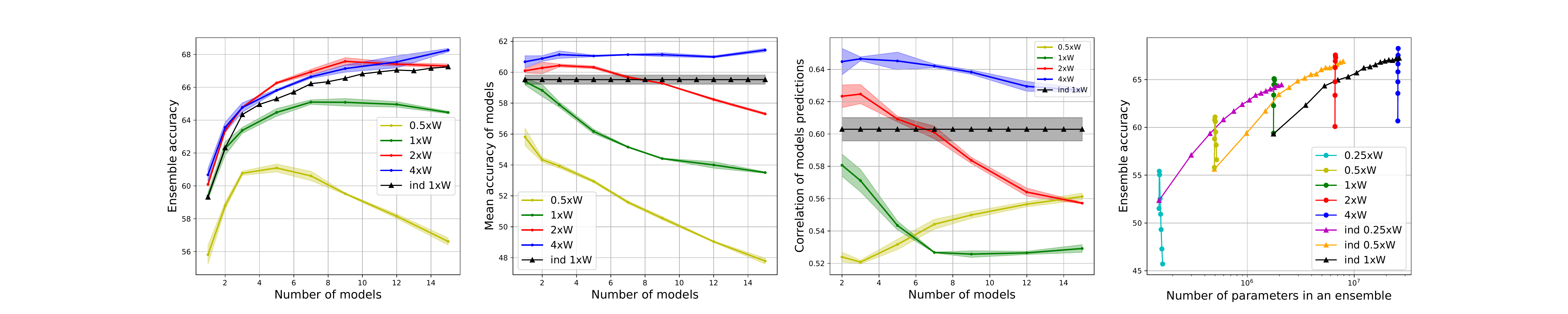}}
    \put(55,110){(a)}
    \put(175,110){(b)}
    \put(300,110){(c)}
    \put(415,110){(d)}
    \end{picture}
    
    \caption{Performance of BatchEnsembles in independent regime and usual independent ensembles for different model widths and numbers of models. Each curve corresponds to a particular model width. The respective value ($0.25\times$ -- $4\times$) is the number of neurons in each layer relative to the baseline network. \textbf{(a), (b), (c):} Same metrics as in the first, second and forth plots of Figure~\ref{fig:shifting_modulations}. \textbf{(d):} Scatter plot comparing various ensembles with respect to accuracy and the absolute number of parameters.}
  \label{fig:many_widths}
  \end{figure*}
 
  \begin{figure*}[tb]
    \centering
    \includegraphics[scale=0.6, trim=3mm 4mm 3mm 3mm, clip]{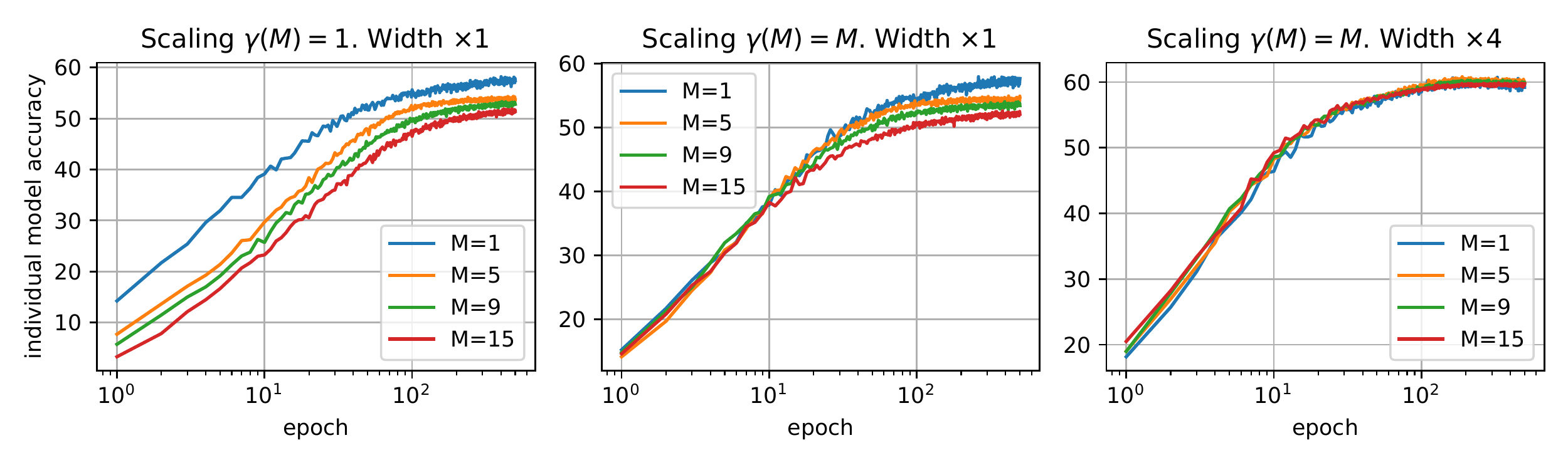}
    \caption{Learning trajectories of individual members of BatchEnsembles, for ensembles of different sizes $M$. The modulations are centered ($\mathbb{E}[u^l_{\alpha}]=0, \mathbb{E}[v^l_{\alpha}]=0$ for all $\alpha,l$). \textbf{Left:} Standard width, learning rate scaled by $\gamma(M)=1$ in equation~\eqref{eq:common_grads}. \textbf{Center:} Standard width, learning rate scaled by $\gamma(M)=M$. \textbf{Right:} Width $\times 4$, learning rate scaled by $\gamma(M)=M$.}
    \label{fig:trajectories_log}
  \end{figure*}
  
  \begin{figure*}[tb]
    \centering
    \setlength{\fboxsep}{0pt}
    {\includegraphics[scale=0.55, trim=3mm 40mm 2mm 4mm, clip]{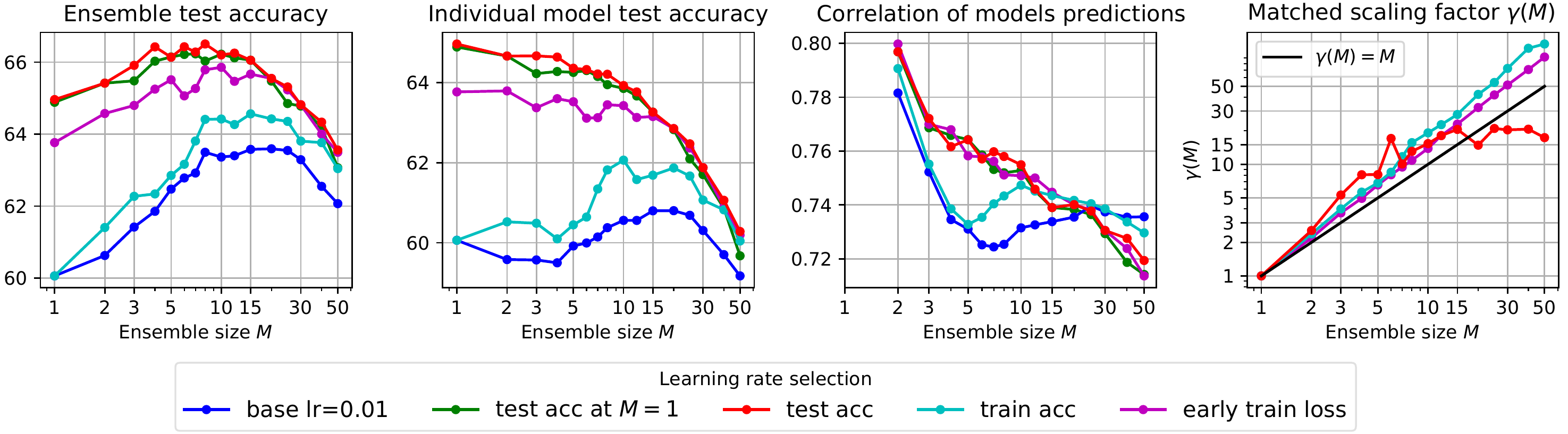}}
    \caption{Performance of LLD ensemble with learning rate $\eta(M)$ at each ensemble size $M$ chosen according to certain criterion specified in the legend. \emph{Blue} lines: constant $\eta(M)=0.01$. \emph{Green}:  constant $\eta(M)$ maximizing test accuracy of single reference model. \emph{Red} and \emph{cyan}: optimal test and train accuracy at the end of training. \emph{Purple}: optimal train loss at early epoch $n=10$. In particular, blue and green lines effectively use simple gradient scaling $\gamma(M)=M$ while the rest of the lines effectively use ``experimental'' scaling $\gamma(M)=M\frac{\eta(M)}{\eta(1)}$, depicted on the right plot.}
    \label{fig:dropout_last}
  \end{figure*}

  \begin{figure*}[tb]
    \centering
    \setlength{\fboxsep}{0pt}
    \includegraphics[scale=0.6, trim=3mm 5mm 2mm 4mm, clip]{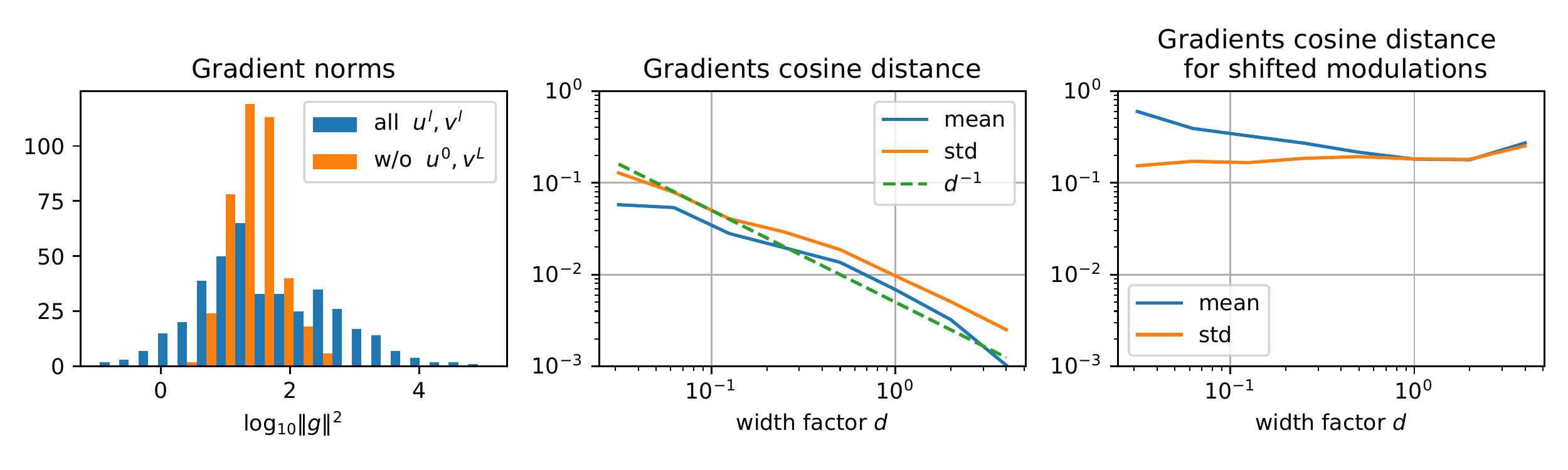}
    \caption{\textbf{Left:} Gradient magnitudes with or without modulations in the initial and final layers. \textbf{Center:} Mean cosine distance $\tfrac{|\langle\nabla\mathcal L_\alpha,\nabla\mathcal L_\beta\rangle|}{\|\nabla\mathcal L_\alpha\|\cdot\|\nabla\mathcal L_\beta\|}$ between common parameters gradients corresponding to different ensemble models, for \emph{centered} modulations ($\mathbb{E}[u^l_{\alpha}]=0, \mathbb{E}[v^l_{\alpha}]=0$ for all $\alpha,l$). \textbf{Right:} Mean cosine distance between model gradients for \emph{shifted} modulations ($\mathbb{E}[u^l_{\alpha}]\ne 0, \mathbb{E}[v^l_{\alpha}]\ne 0$ for all $\alpha,l$).}
  \label{fig:Gradients_at_init}
  \end{figure*}

\subsection{Two operating regimes of embedded ensembles}\label{sec:regimes}
  Theorems~\ref{theorem_1} and~\ref{theorem_2} show that for certain choices of reference network architecture and parameter initialization the covariance and NTK of infinitely wide EEs are model diagonal, which ensures independence of ensemble models. We refer to such EEs as operating in the \textit{independent regime}. If at least some elements of the NTK $\Theta_{\alpha\beta}(\mathbf{x},\mathbf{x}')$ have a non-vanishing value at $\alpha\ne\beta$, we refer to the corresponding EE as operating in the \textit{collective regime}. Examples of EEs in the independent regime are BatchEnsemble with all modulations ${u^l_\alpha,v^l_\alpha}$ initialised as standard Gaussians and MC dropout ensemble with either standard Gaussian or symmetric $\{-1,0,1\}$ masks. On the other hand, we expect BatchEnsembles with only pre-activation modulations $v^l_\alpha$ or MC dropout ensembles with Bernoulli $\{0,1\}$ masks to be in the collective regime.    

  In the \emph{independent regime} the dynamics of any ensemble model $f_\alpha(\mathbf{x},t)$ is completely identical to that of isolated reference model $f(\mathbf{x},t)$ (i.e., the trivial ensemble with $M=1$) if the gradient scaling factor $\gamma(M)=M$. This can be seen from~\eqref{eq:EE_NTK0} and the fact that  covariance $\mathbb{E}[f_\alpha(\mathbf{x})f_\alpha(\mathbf{x}')]$ and NTK $\Theta_{\alpha\alpha}(\mathbf{x},\mathbf{x}')$ of model $\alpha$ are trivially equal to their counterparts for isolated reference network. Thus EE in the independent regime is equivalent to the usual independent ensemble of reference networks. From a practical point of view, it also means that a hyperparameters choice effective for training a single reference network will remain effective for training embedded ensemble of any size $M$. 

  In the \textit{collective regime} we observe that non-zero off-diagonal elements $\Theta_{\alpha\beta}(\mathbf{x},\mathbf{x}')$ of the NTK are  the same for all pairs $(\alpha, \beta)$ of different models due to their equivalence in terms of parameter distributions at initialization. Thus, we can expect the sum over $\beta$ in~\eqref{eq:EE_outputs_dynamic} to scale as $M$ times the typical size of a single term in the sum. This suggests that the natural choice for gradients scaling is $\gamma(M)=1$, which keeps the model evolution $\frac{df_\alpha(\mathbf{x})}{dt}$ bounded for large ensemble sizes $M$. 

  Note that while distinction between {independent} and {collective} regimes is formulated for infinitely wide reference networks, we expect the behavior of finite size EEs to be significantly influenced by the operating regime of their infinite width counterparts. In this section \emph{independent} regime was defined in a narrow sense when interaction between ensemble models strictly zero. However, for finite networks this definition is meaningless, but we still observe two different behavior patterns (see Figure~\ref{fig:shifting_modulations}) which can be naturally identified with collective regime and \emph{independent} regime.

\section{Experiments}
\label{sec:analysis}
We present a series of experiments confirming our theory and further illuminating various properties of embedded ensembles. We perform all experiments on CIFAR100 using a simple 4-hidden-layer convolutional architecture  similar to VGG16~\cite{simonyan2014very}. All experiments except for those in Figures~\ref{fig:shifting_modulations} and~\ref{fig:dropout_last} involve BatchEnsemble with modulations present at all layers and initialized as standard Gaussians $\mathcal{N}(0,1)$. For simplicity, the gradient scaling factor was set to its theoretical value for independent ensemble $\gamma(M)=M$. We measure correlation between ensemble models as standard Pearson correlation coefficient between binarized outputs ``wrong class''/``correct class'', averaged over all model pairs $(\alpha\ne\beta)$. We find this version of correlations to be less dependent on model accuracy. See Section~\ref{sec:experiment_details} for further experimental details and Section~\ref{sec:add_exp} for additional experiments.

\paragraph{Transition between the independent and collective regimes.}\label{sec:transition between regimes}
Consider modulations in EE all initialized with the shifted and scaled Gaussian distribution $\mathcal{N}(p, 1-p^2)$. At $p=1$ the modulations are trivial and deterministic, $u_\alpha^l=v_\alpha^l=1$, and thus all EE submodels are identical, which is an extreme case of collective regime. At $p=0$, infinitely wide EEs are in the independent regime.  In Figure~\ref{fig:shifting_modulations} we gradually change the shift parameter $p$ and plot dependence of various metrics on ensemble size $M$. In this way we check whether the independent and collective regimes can be observed in finite width networks and study the transition between these regimes. %The gradient scaling factor was set to its independent value $\gamma(M)=M$.

We see two distinct groups of models corresponding to shift parameters $p \geq 0.65$ and $p \leq 0.55$ and clearly separated by all four metrics.  We identify these two groups as being in the collective and independent regime, respectively. The differences between regimes for each observed metric are as follows.
\textbf{1)} Ensemble accuracy in the collective regime has a relatively big variance across different $p$, but monotonically grows on average. In the independent regime the variance is smaller and its average has a distinctive ``inverted-U'' shape with best accuracy at $M=7$.
\textbf{2)} Mean accuracy of individual ensemble models on the test set increases (resp. decreases) with ensemble size in the collective (resp. independent) regime. Variance is again higher in the collective regime.
\textbf{3)} Mean accuracy of individual ensemble models on the training set is maximal for ensembles of all sizes in the collective regime, while decreasing monotonically and significantly in the independent regime. 
\textbf{4)} Correlation between prediction of different models within ensemble, as expected, grows with $p$ in the collective regime. In contrast, in the independent regime the correlation is almost constant for the whole $p$ range. 

All observations above confirm the following natural pictures of the two regimes. \textbf{Collective regime:} submodels in the ensemble can be viewed as a sample of models ``in a vicinity of one good model'' and thus having low diversity. This picture explains all observed features of the collective regime, except for the growth of accuracy of individual models on the test set with $M$. This exception is an artifact of our experimental setup with independent gradient scaling $\gamma(M)=M$ which results in an effectively higher learning rate for larger ensembles. \textbf{Independent regime:} submodels in the ensemble are relatively diverse, have a good generalization ability (small difference between train and test accuracy) and thus provide a significant accuracy gain when ensembled together. The low variance of all four metrics in the independent regime can be explained by averaging: the variance of a metric (e.g. mean test accuracy) is reduced by a factor of $M$ compared to that of isolated model, while in the collective regime we have an ``approximately single'' model and thus almost no variance reduction.

The transition between regimes is sharp: EEs with $p=0.55$ and $0.65$ already have a distinctive behaviour of the respective regimes. The EE at the transition point $p\approx 0.6$ exhibits ``critical'' behavior -- the transition  is driven not by the initialization parameter $p$ but by ensemble size $M$ and happens at $M\approx10$.   
  
\paragraph{Scaling of EE with the width of reference network.}
  In Figure~\ref{fig:many_widths} we show the performance of usual independent ensembles and BatchEnsembles in the individual regime with different numbers of models and rescaled network widths. We observe that the optimal ensemble size maximizing ensemble accuracy increases with network width. Such behavior agrees with our theoretical results (Theorems~\ref{theorem_1} and~\ref{theorem_2}) for the infinite network width limit, where ensemble models become truly independent and ensemble accuracy grows monotonically with $M$.   
  
  In Figure~\ref{fig:many_widths}-(d) we plot ensemble accuracy against the total number of its parameters and observe that BatchEnsembles with the optimal number of models slightly outperform independent ensembles at any fixed number of parameters.

\paragraph{Gradient scaling factor $\gamma(M)$.}
In Figure~\ref{fig:trajectories_log} we show learning trajectories of individual submodels of BatchEnsemble in the independent regime. By comparing the left and center subfigures, we see that it is indeed natural and beneficial to use the gradient scaling factor $\gamma(M)=M$ in the averaged gradient~\eqref{eq:common_grads}. With this choice, learning trajectories are almost identical in the early epochs for all $M$, implying that models train independently. According to the theory, the length of the initial time period of independent training should increase with model width, and we indeed observe this effect by comparing the right and central subfigures.

The full matching of trajectories of individual models $f_\alpha(\mathbf{x},t)$ and isolated reference model $f(\mathbf{x},t)$ can be achieved only in the fully independent regime (e.g. at infinite network width), as discussed in Section~\ref{sec:regimes}. On the other hand, for finite networks we can match ensembles of different sizes only based on specific criterion, e.g. maximal accuracy at the end of training, either on test or train test. Thus the resulting scaling $\gamma(M)$ generally depends on the criterion used for matching. In Figure~\ref{fig:dropout_last} we perform such matching by first training the EE at different learning rates $\eta$, then obtaining optimal learning rates $\eta^{(M)}$ according to a certain criterion, and finally calculating the gradient scaling as $\gamma(M)=M\frac{\eta^{(M)}}{\eta^{(1)}}$. The results are depicted in the last plot of Figure~\ref{fig:dropout_last}, where we see that corrections to simple scaling $\gamma(M)=M$ are moderate: $|\log(\frac{\gamma(M)}{M})| < \log 3$ for $M<50$.  

\paragraph{Last layer dropout (LLD) ensemble.}
LLD ensembles described in Section~\ref{sec:EE_description} have computation complexity of inference and backpropagation almost as small as a single model ($\sim 10\%$ increase in computation time for ensemble of size $M=50$). Accordingly, we were able to train ensembles of larger sizes $M$ and on a wide range of learning rates $\eta$ from $10^{-3}$ to $10^{1}$. This grid search allows to find optimal values of learning rate for each $M$, and effectively determine experimental values of gradient scaling factor $\gamma(M)$. The modulations in this experiment were standard Gaussian and thus we expect it to be in independent regime. The results are shown in Figure~\ref{fig:dropout_last}.

First, observe that optimizing learning rate for accuracy on test set (red line) leads to significant increase of ensemble accuracy compared to basic learning rate (blue line). However, simpler learning rate optimization for only single model (green line) gives almost optimal performance. For all learning rate selections we observe that ensemble accuracy of LLD has independent regime behaviour with optimal number of models $M\approx10$. However, the increase in accuracy $\sim 2\%$ at optimal size  is less than the respective increase $\sim 5\%$ for BatchEnsemble with modulations at all layers. Thus the presence of these modulations does improve model independence for finite width networks while being unnecessary for infinite width networks.         
  
\paragraph{Orthogonality of gradients for different submodels.}
In Figure~\ref{fig:Gradients_at_init} (center and right) we confirm our theoretical analysis of the presence of dynamical model independence (Theorem~\ref{theorem_2}) by checking the orthogonality of loss gradients for different submodels. We find the orthogonality to hold if the network is wide enough and the post-activation modulations are centered ($\mathbb{E}[u^l_{\alpha}]=0$); otherwise, the gradients are not orthogonal in general. In the centered case, both the average and standard deviation of the normalized scalar product  $\tfrac{|\langle\nabla\mathcal L_\alpha,\nabla\mathcal L_\beta\rangle|}{\|\nabla\mathcal L_\alpha\|\cdot\|\nabla\mathcal L_\beta\|}$ scales with the relative network width $d$ as $\sim d^{-1}$.

Also, we studied the effect of adding input modulations $u^0_\alpha$ and output modulations $v^{L+1}$ present in the original BatchEnsemble architecture~\cite{wen2019batchensemble}. In our experiments we find that it is considerably harder to train BatchEnsemble with these modulations and observe frequent occurrences of training divergence. In Figure~\ref{fig:Gradients_at_init} (left) we plot the histogram of gradient norms w.r.t. random initialization of network parameters. Several orders of magnitude increase of the gradients spread in the architecture with input and output modulations explains its unstable behaviour.  
  
\section{Conclusion}
\label{sec:conclusion}
  We have performed an extensive theoretical and empirical study of embedded ensembles. Our main finding is the identification of two quite distinct \emph{independent} and \emph{collective} training regimes in which the ensemble models are approximately independent of each other or rather look like perturbations of a single model, respectively. The existence of these regimes was proved theoretically in the infinite-width network limit and verified experimentally for finite networks. The two regimes have quite different properties: in particular, in the independent regime performance is maximized at a particular finite number of models in the ensemble while in the collective regime it increases indefinitely with the size of the ensemble. The optimal performance seems to be achieved by ensembles of optimal size in the independent regime. The optimal size increases with wider networks. Our theory also suggests a particular optimal scaling for the ensemble learning rate, which we have confirmed experimentally.
  
  Post-activation modulations in the last hidden layer play a special role in the ensemble; even if only these modulations are present, the ensemble shows all its characteristic properties. This special case is especially attractive for practical applications since in this case not only storage but also computation time is substantially reduced compared to the standard ensemble of independent models. 
  
  We remark that in this paper we have covered BatchEnsembles and (a version of) MC dropout ensembles, but not other embedded ensembles. Some of them require significant modifications of our theoretical approach. For example, our NTK-based theory assumes that the weights do not significantly change under training, which is not the case for MIMO ensembles~\cite{havasi2020training}, since some of their weights are expected to vanish as a result of training. We leave analysis of other embedded ensembles to a future work.

\section*{Acknowledgment}
Research was supported by Russian Science Foundation, grant 21-11-00373. The authors acknowledge the use of computational resources of the Skoltech CDISE supercomputer Zhores~\cite{Zhores2019} for obtaining the results presented in this paper.
 
\bibliography{refs}

\begin{thebibliography}{}

\bibitem[Alemohammad et~al., 2021]{Alemohammad2021TheRN}
Alemohammad, S., Wang, Z., Balestriero, R., and Baraniuk, R. (2021).
\newblock The recurrent neural tangent kernel.
\newblock In {\em International Conference on Learning Representations 2021}.

\bibitem[Arora et~al., 2019]{Arora2019OnEC}
Arora, S., Du, S.~S., Hu, W., Li, Z., Salakhutdinov, R., and Wang, R. (2019).
\newblock On exact computation with an infinitely wide neural net.
\newblock In {\em NeurIPS 2019}.

\bibitem[Dietterich, 2000]{dietterich2000ensemble}
Dietterich, T.~G. (2000).
\newblock Ensemble methods in machine learning.
\newblock In {\em International workshop on multiple classifier systems}, pages
  1--15. Springer.

\bibitem[Durasov et~al., 2021]{durasov2021masksembles}
Durasov, N., Bagautdinov, T., Baque, P., and Fua, P. (2021).
\newblock Masksembles for uncertainty estimation.
\newblock In {\em Proceedings of the IEEE/CVF Conference on Computer Vision and
  Pattern Recognition}, pages 13539--13548.

\bibitem[Dusenberry et~al., 2020]{dusenberry2020efficient}
Dusenberry, M., Jerfel, G., Wen, Y., Ma, Y., Snoek, J., Heller, K.,
  Lakshminarayanan, B., and Tran, D. (2020).
\newblock Efficient and scalable {B}ayesian neural nets with rank-1 factors.
\newblock In {\em International conference on machine learning}, pages
  2782--2792. PMLR.

\bibitem[Dyer and Gur-Ari, 2020]{Dyer2020AsymptoticsOW}
Dyer, E. and Gur-Ari, G. (2020).
\newblock Asymptotics of wide networks from feynman diagrams.
\newblock In {\em ICLR 2020}.

\bibitem[Gal, 2016]{gal2016}
Gal, Y. (2016).
\newblock Uncertainty in deep learning.
\newblock {\em University of Cambridge}.

\bibitem[Gal and Ghahramani, 2016]{gal2016dropout}
Gal, Y. and Ghahramani, Z. (2016).
\newblock Dropout as a bayesian approximation: Representing model uncertainty
  in deep learning.
\newblock In {\em international conference on machine learning}, pages
  1050--1059. PMLR.

\bibitem[Garipov et~al., 2018]{garipov2018loss}
Garipov, T., Izmailov, P., Podoprikhin, D., Vetrov, D.~P., and Wilson, A.~G.
  (2018).
\newblock Loss surfaces, mode connectivity, and fast ensembling of dnns.
\newblock In {\em Advances in Neural Information Processing Systems}, pages
  8789--8798.

\bibitem[Hanin and Nica, 2020]{Hanin2020FiniteDA}
Hanin, B. and Nica, M. (2020).
\newblock Finite depth and width corrections to the neural tangent kernel.
\newblock In {\em International Conference on Learning Representations 2020}.

\bibitem[Havasi et~al., 2020]{havasi2020training}
Havasi, M., Jenatton, R., Fort, S., Liu, J.~Z., Snoek, J., Lakshminarayanan,
  B., Dai, A.~M., and Tran, D. (2020).
\newblock Training independent subnetworks for robust prediction.
\newblock In {\em International Conference on Learning Representations 2020}.

\bibitem[He et~al., 2020]{He2020BayesianDE}
He, B., Lakshminarayanan, B., and Teh, Y.~W. (2020).
\newblock Bayesian deep ensembles via the neural tangent kernel.
\newblock In {\em Advances in Neural Information Processing Systems},
  volume~33, pages 1010--1022.

\bibitem[Hron et~al., 2020]{Hron2020InfiniteAN}
Hron, J., Bahri, Y., Sohl-Dickstein, J., and Novak, R. (2020).
\newblock Infinite attention: Nngp and ntk for deep attention networks.
\newblock In {\em ICML 2020}.

\bibitem[Huang et~al., 2017]{huang2017snapshot}
Huang, G., Li, Y., Pleiss, G., Liu, Z., Hopcroft, J.~E., and Weinberger, K.~Q.
  (2017).
\newblock Snapshot ensembles: Train 1, get m for free.
\newblock In {\em ICLR 2017}.

\bibitem[Huang and Yau, 2020]{Huang2020DynamicsOD}
Huang, J. and Yau, H.-T. (2020).
\newblock Dynamics of deep neural networks and neural tangent hierarchy.
\newblock In {\em International conference on machine learning}, pages
  4542--4551. PMLR.

\bibitem[Jacot et~al., 2018]{jacot2018neural}
Jacot, A., Gabriel, F., and Hongler, C. (2018).
\newblock Neural tangent kernel: Convergence and generalization in neural
  networks.
\newblock In {\em Advances in neural information processing systems 2018}.

\bibitem[Lee et~al., 2018]{lee2018deep}
Lee, J., Bahri, Y., Novak, R., Schoenholz, S.~S., Pennington, J., and
  Sohl-Dickstein, J. (2018).
\newblock Deep neural networks as gaussian processes.
\newblock In {\em International Conference on Learning Representations 2018}.

\bibitem[Lee et~al., 2020]{Lee2020FiniteVI}
Lee, J., Schoenholz, S., Pennington, J., Adlam, B., Xiao, L., Novak, R., and
  Sohl-Dickstein, J. (2020).
\newblock Finite versus infinite neural networks: an empirical study.
\newblock In {\em Advances in Neural Information Processing Systems},
  volume~33, pages 15156--15172.

\bibitem[Lee et~al., 2019]{lee2019wide}
Lee, J., Xiao, L., Schoenholz, S., Bahri, Y., Novak, R., Sohl-Dickstein, J.,
  and Pennington, J. (2019).
\newblock Wide neural networks of any depth evolve as linear models under
  gradient descent.
\newblock In {\em Advances in neural information processing systems},
  volume~32.

\bibitem[Lewkowycz et~al., 2020]{Lewkowycz2020TheLL}
Lewkowycz, A., Bahri, Y., Dyer, E., Sohl-Dickstein, J., and Gur-Ari, G. (2020).
\newblock The large learning rate phase of deep learning: the catapult
  mechanism.
\newblock {\em ArXiv}, abs/2003.02218.

\bibitem[Liu et~al., 2020]{Liu2020OnTL}
Liu, C., Zhu, L., and Belkin, M. (2020).
\newblock On the linearity of large non-linear models: when and why the tangent
  kernel is constant.
\newblock {\em Advances in Neural Information Processing Systems},
  33:15954--15964.

\bibitem[Ram{\'e} et~al., 2021]{rame2021mixmo}
Ram{\'e}, A., Sun, R., and Cord, M. (2021).
\newblock Mixmo: Mixing multiple inputs for multiple outputs via deep
  subnetworks.
\newblock In {\em Proceedings of the IEEE/CVF International Conference on
  Computer Vision}, pages 823--833.

\bibitem[Schwenk and Bengio, 2000]{schwenk2000boosting}
Schwenk, H. and Bengio, Y. (2000).
\newblock Boosting neural networks.
\newblock {\em Neural computation}, 12(8):1869--1887.

\bibitem[Seoh, 2020]{seoh2020qualitative}
Seoh, R. (2020).
\newblock Qualitative analysis of monte carlo dropout.
\newblock {\em arXiv}, abs/2007.01720.

\bibitem[Sicking et~al., 2020]{Sicking2020DropoutCharacteristics}
Sicking, J., Akila, M., Wirtz, T., Houben, S., and Fischer, A. (2020).
\newblock Characteristics of monte carlo dropout in wide neural networks.
\newblock {\em CoRR}, abs/2007.05434.

\bibitem[Simonyan and Zisserman, 2015]{simonyan2014very}
Simonyan, K. and Zisserman, A. (2015).
\newblock Very deep convolutional networks for large-scale image recognition.
\newblock In {\em ICLR 2015}.

\bibitem[Sohl-Dickstein et~al., 2020]{sohl2020infinite}
Sohl-Dickstein, J., Novak, R., Schoenholz, S.~S., and Lee, J. (2020).
\newblock On the infinite width limit of neural networks with a standard
  parameterization.
\newblock {\em arXiv preprint arXiv:2001.07301}.

\bibitem[Srivastava et~al., 2014]{JMLR:v15:srivastava14a}
Srivastava, N., Hinton, G., Krizhevsky, A., Sutskever, I., and Salakhutdinov,
  R. (2014).
\newblock Dropout: A simple way to prevent neural networks from overfitting.
\newblock {\em Journal of Machine Learning Research}, 15(56):1929--1958.

\bibitem[Verdoja and Kyrki, 2021]{Verdoja2021Dropout}
Verdoja, F. and Kyrki, V. (2021).
\newblock Notes on the behavior of mc dropout.
\newblock In {\em ICML Workshop on Uncertainty \& Robustness in Deep Learning}.

\bibitem[Wen et~al., 2020]{wen2020improving}
Wen, Y., Jerfel, G., Muller, R., Dusenberry, M.~W., Snoek, J.,
  Lakshminarayanan, B., and Tran, D. (2020).
\newblock Improving calibration of batchensemble with data augmentation.
\newblock In {\em ICML 2020 workshop on Uncertainty Robustness in Deep
  Learning}.

\bibitem[Wen et~al., 2019]{wen2019batchensemble}
Wen, Y., Tran, D., and Ba, J. (2019).
\newblock Batchensemble: an alternative approach to efficient ensemble and
  lifelong learning.
\newblock In {\em International Conference on Learning Representations}.

\bibitem[Wenzel et~al., 2020]{wenzel2020hyperparameter}
Wenzel, F., Snoek, J., Tran, D., and Jenatton, R. (2020).
\newblock Hyperparameter ensembles for robustness and uncertainty
  quantification.
\newblock In {\em Advances in Neural Information Processing Systems},
  volume~33, pages 6514--6527.

\bibitem[{Zacharov} et~al., 2019]{Zhores2019}
{Zacharov}, I., {Arslanov}, R., {Gunin}, M., {Stefonishin}, D., {Bykov}, A.,
  {Pavlov}, S., {Panarin}, O., {Maliutin}, A., {Rykovanov}, S., and {Fedorov},
  M. (2019).
\newblock Zhores - petaflops supercomputer for data-driven modeling, machine
  learning and artificial intelligence installed in skolkovo institute of
  science and technology.
\newblock In {\em Open Engineering}, volume 9(1), pages 512--520.

\bibitem[Zhang et~al., 2021]{zhang2021ex}
Zhang, Z., Gao, V.~R., and Sabuncu, M.~R. (2021).
\newblock Ex uno plures: Splitting one model into an ensemble of subnetworks.
\newblock {\em arXiv preprint arXiv:2106.04767}.

\end{thebibliography}

\clearpage
\appendix

\thispagestyle{empty}

\onecolumn \makesupplementtitle

\section{Background on NTK}\label{sec:NTK_background}
\paragraph{Neural Tangent Kernel.} Consider a general parametric model $f(\mathbf{x},\mathbf{W})$ with inputs $\mathbf{x}$ and parameters $\mathbf{W}$. Also assume that parameters $\mathbf{W}$ are trained with continuous time gradient decent on a train set $\mathcal{D}=\{(\mathbf{x}_a,y_a\}_{a=1}^B$ of size $B$ and with a loss 
\begin{equation}
    \mathcal{L}(\mathbf{W})=\frac{1}{B}\sum_{(\mathbf{x},y)\in\mathcal{D}}L(f(\mathbf{x},\mathbf{W}), y).
\end{equation}
The evolution of model parameters $\mathbf{W}$ is then given by $    \frac{d}{dt}\mathbf{W}(t) = -\frac{\partial \mathcal{L}(\mathbf{W})}{\partial \mathbf{W}}$. This evolution induces the following evolution of network outputs  
\begin{align}\label{eq:param_model_output_evolution}
    \frac{d}{dt}f(\mathbf{x}, \mathbf{W}) &= -\frac{\partial f(\mathbf{x}, \mathbf{W})}{\partial \mathbf{W}}\frac{\partial \mathcal{L}(\mathbf{W})}{\partial \mathbf{W}} = - \frac{1}{B} \sum\limits_{b = 1}^B \frac{\partial f(\mathbf{x}, \mathbf{W})}{\partial \mathbf{W}} \frac{\partial f(\mathbf{x}_b, \mathbf{W})}{\partial \mathbf{W}} \frac{\partial L(f(\mathbf{x}_b), y_b)}{\partial f(\mathbf{x}_b)} \\
    \label{eq:NTK_model_output_evolution}
    &= -\frac{1}{B} \sum\limits_{b = 1}^B \Theta(\mathbf{x}, \mathbf{x}_b, \mathbf{W}_n) \frac{\partial L(f(\mathbf{x}_b), y_b)}{\partial f(\mathbf{x}_b)}.
\end{align}
Here the product of two partial derivatives w.r.t. model parameters is understood as the inner product of respective gradients. In the last equality we have replaced inner product of model output gradients with $\Theta(\mathbf{x},\mathbf{x}_a, \mathbf{W})$ -- the \textit{Neural Tangent Kernel (NTK)} of model $f(\mathbf{x},\mathbf{W})$ \cite{jacot2018neural}. The definition of NTK on an arbitrary pair of inputs is
\begin{equation}
    \Theta(\mathbf{x},\mathbf{x}',\mathbf{W}) \equiv \frac{\partial f(\mathbf{x}, \mathbf{W})}{\partial \mathbf{W}} \frac{\partial f(\mathbf{x}', \mathbf{W})}{\partial \mathbf{W}}.
\end{equation}
Note that we can write a similar equation for the evolution of model outputs in case of discrete gradient descent with learning rate $\eta$. However, it will be an approximate equation since it will rely on first order Taylor expansion of $f(\mathbf{x}, \mathbf{W} - \eta \nabla_\mathbf{W}\mathcal{L}(\mathbf{W}))$: 
\begin{equation}
    f(\mathbf{x}, \mathbf{W}_{n+1}) - f(\mathbf{x}, \mathbf{W}_{n}) = -\eta \frac{1}{B} \sum\limits_{b=1}^B \Theta(\mathbf{x},\mathbf{x}_b, \mathbf{W}_n) \frac{\partial L(f(\mathbf{x}_b), y_b)}{\partial f(\mathbf{x}_b)} + O(\eta^2).
\end{equation}
For this reason, in the theoretical part of the paper we consider only the continuous time gradient descent. In the remaining parts of Supplementary Material we will exclude $\mathbf{W}$ from the notation for model output $f(\mathbf{x})$ and NTK $\Theta(\mathbf{x}, \mathbf{x}')$ assuming they depend on model parameters implicitly. 

\paragraph{``Standard'' and ``NTK'' parametrizations.} Now consider a depth-$L$ fully-connected neural network parametrized as 
\begin{equation}\label{eq:fcc_layer}
    \begin{cases}
    &z^{l}_{j} = \frac{1}{\sqrt{N_{l-1}}} \sum\limits_{i = 1}^{N_{l - 1}} x_{i}^{l-1} w^{l}_{ij} + b^l_j, \\
    &x^{l}_{j} = \phi\bigl(z^l_{j}\bigr).
    \end{cases}
\end{equation}
Here $w^l_{ij}, b^l_j$ are model parameters; $N_l, \; l=1,\ldots,L,$ are the widths of hidden layers; $x^0_j$ is the $j$'th component of input $\mathbf{x}$ and $z^{L+1}_j$ is the $j$'th component of network output $f(\mathbf{x})$. An important feature of parametrization~\eqref{eq:fcc_layer} is the explicit dependence on the layer widths $N_l$ in the network forward pass, while the weights are initialized as standard Gaussian $w^l_{ij} \sim \mathcal{N}(0,1)$. Such initialization is equivalent to the ``standard'' one (no explicit dependence on $N_l$ in the network forward pass but weights are initialized as $w^l_{ij}\sim \mathcal{N}(0,1/N_{l-1})$) in terms of distribution of outputs at initialization, but leads to different gradient descent dynamic. In particular, for the parametrization~\eqref{eq:fcc_layer} in the infinite width limit $N_l\rightarrow \infty$ the NTK is finite and the gradient descent dynamic is non-divergent, while ``standard'' parametrization requires learning rates of layer $l$ parameters to be scaled as $\eta_l \sim \frac{1}{N_{l-1}}$ for the gradient descent dynamic to be non-divergent. For this reason the parametrization~\eqref{eq:fcc_layer} is often called ``NTK'' parametrization, and it is chosen for the current work.

\paragraph{Network evolution in the infinite-width limit.} Now we further consider infinite width limit $N_l\rightarrow\infty$, where the network is characterized by the following key properties (see Section~\ref{sec:output_dyn}):
\begin{itemize}
    \item[(A1)] The network output $f(\mathbf{x})$ is a draw from a Gaussian Process (GP) at initialization.
    \item[(A2)] The NTK $\Theta(\mathbf{x}, \mathbf{x}')$ converges to a non-random deterministic value at initialization.
    \item[(A3)] The NTK $\Theta(\mathbf{x}, \mathbf{x}')$ stays constant during training.
\end{itemize}
The properties (A2, A3) actually allow to forget about model parameters $\mathbf{W}$ in~\eqref{eq:NTK_model_output_evolution} since the NTK has some deterministic predefined values, and derivatives of the loss depend only on network outputs. This means that~\eqref{eq:NTK_model_output_evolution} becomes a closed equation on model outputs $f(\mathbf{x})$ and can be integrated to obtain the model outputs $f_t(\mathbf{x})$ at time $t$ given their initial values $f_0(\mathbf{x})$. In the case of regression problem with MSE loss $L(f,y)=\frac{1}{2}(f-y)^2$ the output dynamics can be integrated analytically
\begin{equation}
    \delta \mathbf{f}_t = \delta\mathbf{f}_0 e^{-\mathbf{\Theta} t}.
\end{equation}
Here the $\delta \mathbf{f}_t$ is a size $B$ vector of network approximation errors $\delta f_{a,t}\equiv f_t(\mathbf{x}_a)-y_a$ on the train set $\mathcal{D}$ and $\mathbf{\Theta}$ is a $B\times B$ NTK matrix $\Theta_{ab}\equiv \Theta(\mathbf{x}_a, \mathbf{x}_b)$.

\paragraph{Classification with cross-entropy loss.} Now let us consider in more detail the case of classification problem with $C$ classes and cross entropy loss
\begin{equation}\label{eq:cross_entropy}
    L(\mathbf{f}, c) = -\log p_c(\mathbf{f}), \quad p_i(\mathbf{f})\equiv \frac{\exp(f_i)}{\sum_{j=1}^C\exp(f_j)}, 
\end{equation}
where $c$ is the target class and $\mathbf{f}$ is the vector of model outputs. We will rely on this discussion when describing experiments with NTK at initialisation and during training in Section~\ref{sec:NTK_init_and_dyn}. In the described setting the network output will have $C$ components and corresponding NTK will have additional matrix structure in the class space
\begin{equation}
\Theta_{ii'}(\mathbf{x},\mathbf{x}') \equiv \frac{\partial f_i(\mathbf{x}, \mathbf{W})}{\partial \mathbf{W}} \frac{\partial f_{i'}(\mathbf{x}', \mathbf{W})}{\partial \mathbf{W}}.    
\end{equation}
Let us now derive the expression for the time derivative of train set loss $\mathcal{L}$
\begin{equation}
\begin{split}
    \frac{d \mathcal{L}}{dt} &= -\frac{1}{B} \sum_a \frac{\partial L(\mathbf{f}(\mathbf{x}_a), c_a)}{\partial \mathbf{f}(\mathbf{x}_a)} \frac{d\mathbf{f}(\mathbf{x}_a)}{dt} \\
    &= - \frac{1}{B^2}\sum_{a,b}\sum_{i,j} \frac{\partial L(\mathbf{f}(\mathbf{x}_a), c_a)}{\partial f_i(\mathbf{x}_a)} \Theta_{ij}(\mathbf{x}_a, \mathbf{x}_b) \frac{\partial L(\mathbf{f}(\mathbf{x}_b), c_b)}{\partial f_j(\mathbf{x}_b)} = -\left(\frac{1}{B}\nabla_\mathbf{f}\mathcal{L}\right)^T  \mathbf{\Theta} \left(\frac{1}{B}\nabla_\mathbf{f}\mathcal{L}\right).
\end{split}
\end{equation}
Here in the last equality we denoted by $\nabla_\mathbf{f}\mathcal{L}$ the $BC$-sized vector of with components $\frac{\partial L(\mathbf{f}(\mathbf{x}_a), c_a)}{\partial f_i(\mathbf{x}_a)}$ and by $\mathbf{\Theta}$ the respective $BC\times BC$ NTK matrix in order to emphasize inner product structure of right-hand side. For the cross entropy loss $\nabla_\mathbf{f}\mathcal{L}$ has components
\begin{equation}\label{eq:CS_loss_derivative_components}
    \frac{\partial L(\mathbf{f}(\mathbf{x}), c)}{\partial f_i(\mathbf{x})} = p_i(\mathbf{f}(\mathbf{x})) - \delta_{ic}.
\end{equation}
For the sufficiently trained network with $p_{c}(\mathbf{f}(\mathbf{x})) \approx 1$ the components of $\nabla_\mathbf{f}\mathcal{L}$ with $i=c_a$ will dominate over components with $i\ne c_a$. This can be understood from~\eqref{eq:CS_loss_derivative_components} by noting that for wrong classes $i\ne c$ the respective probability has the order $p_i \sim \frac{1-p_c}{C-1}$, which is significantly smaller than $1-p_c$ if the number of classes $C$ is big. Thus the NTK components $\Theta_{c_ac_b}(\mathbf{x}_a,\mathbf{x}_b)$ become more important than the other components in the late stages of training.

\section{NTK for Embedded Ensembles}\label{sec:EEs_NTK}
\subsection{Setup}\label{sec:setup}
In this section we focus on fully connected network architecture given by
\begin{equation}
  \label{eq:EE_layer}
    \begin{cases}
      z^{l}_{\alpha j} = \frac{1}{\sqrt{N_{l-1}}} \sum\limits_{i = 1}^{N_{l - 1}} x_{\alpha i}^{l-1} w^{l}_{ij},
      \\
      x^{l}_{\alpha j} = u^{l}_{\alpha j} ~ \phi\bigl(v^{l}_{\alpha j} z^l_{\alpha j}\bigr).
    \end{cases}
\end{equation}
Note that this architecture is identical to the BatchEnsemble architecture given in the main text, except for absence of bias terms. We have decided to exclude the bias terms in order to keep the formulas compact; all the results below can be easily generalized to include bias terms. Although this architecture includes all pre- and post-activation modulations in all layers, we would like to consider cases where modulations of certain types are absent at some layers. To this end we only fix initialisation parameters of common weights $w^l_{ij} \sim \mathcal{N}(0,1)$, and allow all the other parameters $v^l_{\alpha j}, u^l_{\alpha j}$ to have an arbitrary distributions, identical within layer and across models, and having finite first two moments denoted (with a slight abuse of notation) by
\begin{align}
    &\mathbb{E}[v^l_{\alpha j}]=V^l_1, \quad \mathbb{E}[(v^l_{\alpha j})^2]=V^l_2, \\ 
    &\mathbb{E}[u^l_{\alpha j}]=U^l_1, \quad \mathbb{E}[(u^l_{\alpha j})^2]=U^l_2.
\end{align}
In this way we are able to remove any modulation from the architecture by simply considering it a deterministic random variable with the moments  $V^l_1=V^l_2=1$ (implying that this random variable equals 1 a.s.). The MC Dropout ensembles can also be covered by~\eqref{eq:EE_layer} by allowing the modulations to be  nontrainable. We denote whether the corresponding modulation is trainable by the binary function $T(u)$ which equals $1$ when the modulation $u$ is trainable and $0$ otherwise.  

\subsection{Output dynamics}
Here we follow the steps outlined in Section~\ref{sec:NTK_background}. As described in Section~\ref{sec:train_emb_ens}, the EE parameters are updated according to
\begin{align}
    &\frac{d}{dt}\mathbf{w} = - \frac{\gamma(M)}{M} \sum\limits_{\beta=1}^M \frac{\partial \mathcal{L}_\beta(\mathbf{w}, \mathbf{u}_\beta)}{\partial \mathbf{w}}, \\
    &\frac{d}{dt}\mathbf{u}_\alpha=- \frac{\partial \mathcal{L}_\alpha(\mathbf{w}, \mathbf{u}_\alpha)}{\partial \mathbf{u}_\alpha}.
\end{align}
Here we for simplicity ignored different learning rates for common and individual parameters. The induced dynamic for EE model outputs $f_\alpha(\mathbf{x})\equiv f(\mathbf{w},\mathbf{u}_\alpha,\mathbf{x})$ is
\begin{equation}
    \begin{split}
        \frac{d}{dt} f(\mathbf{w},\mathbf{u}_\alpha,\mathbf{x}) =& -\frac{\gamma(M)}{M}\sum\limits_{\beta=1}^M\frac{\partial f(\mathbf{w},\mathbf{u}_\alpha,\mathbf{x})}{\partial \mathbf{w}} \frac{\partial \mathcal{L}_\beta(\mathbf{w},\mathbf{u}_\beta)}{\partial \mathbf{w}} - \frac{\partial f(\mathbf{w},\mathbf{u}_\alpha,\mathbf{x})}{\partial \mathbf{u}_\alpha} \frac{\partial \mathcal{L}_\alpha(\mathbf{w},\mathbf{u}_\alpha)}{\partial \mathbf{u}_\alpha}\\
        =&-\frac{\gamma(M)}{M B}\sum\limits_{\beta=1}^M\sum\limits_{b=1}^B\frac{\partial f(\mathbf{w},\mathbf{u}_\alpha,\mathbf{x})}{\partial \mathbf{w}}\frac{\partial f(\mathbf{w},\mathbf{u}_\beta,\mathbf{x}_b)}{\partial \mathbf{w}} \frac{\partial L(f_\beta(\mathbf{x}_b), y_b)}{\partial f_\beta(\mathbf{x}_b)} \\
        &-\frac{1}{B}\sum\limits_{b=1}^B\frac{\partial f(\mathbf{w},\mathbf{u}_\alpha,\mathbf{x})}{\partial \mathbf{u}_\alpha}\frac{\partial f(\mathbf{w},\mathbf{u}_\alpha,\mathbf{x}_b)}{\partial \mathbf{u}_\alpha} \frac{\partial L(f_\alpha(\mathbf{x}_b), y_b)}{\partial f_\alpha(\mathbf{x}_b)}\\
        &=-\frac{1}{B}\sum\limits_{b=1}^B \Theta_{\alpha\beta}(\mathbf{x},\mathbf{x}_b)\frac{\partial L(f_\beta(\mathbf{x}_b),y_b)}{\partial f_\beta(\mathbf{x}_b)}.
    \end{split}
\end{equation}
Thus we again managed to collect inner products of gradients of EE models outputs in a single quantity, which we naturally call \textit{Neural Tangent Kernel} of the respective Embedded Ensemble. Its full definition is given as in equation~\eqref{eq:EE_NTK0}:
\begin{equation}\label{eq:EE_NTK}
\begin{split}
    &\Theta_{\alpha\beta}(\mathbf{x},\mathbf{x}') = \frac{\gamma(M)}{M}\Theta^{com}_{\alpha\beta}(\mathbf{x},\mathbf{x}') + \delta_{\alpha\beta}\Theta_\alpha^{\text{ind}}(\mathbf{x},\mathbf{x}'),\\
    &\Theta^{\text{com}}_{\alpha\beta}(\mathbf{x},\mathbf{x}') = \frac{\partial f(\mathbf{w}, \mathbf{u}_\alpha,\mathbf{x})}{\partial \mathbf{w}} \frac{\partial f(\mathbf{w}, \mathbf{u}_\beta,\mathbf{x}')}{\partial \mathbf{w}},\\
    &\Theta^{\text{ind}}_\alpha(\mathbf{x},\mathbf{\mathbf{x}}') = \frac{\partial f(\mathbf{w}, \mathbf{u}_\alpha,\mathbf{x})}{\partial \widetilde{\mathbf{u}}_\alpha} \frac{\partial f(\mathbf{w}, \mathbf{u}_\alpha,\mathbf{x}')}{\partial \widetilde{\mathbf{u}}_\alpha}.
\end{split}
\end{equation}

\subsection{The layer-wise hierarchy of covariances and NTK's}\label{sec:SM_detailed_results}

The main  theoretical results of the paper are stated as Theorems~\ref{theorem_1} and~\ref{theorem_2}. The key technical ingredients of these two theorems are recursive descriptions of, respectively, the layer output covariances and NTK's. In this section we state them in the form of respective Propositions~\ref{prop:1} and~\ref{prop:2}.

We first describe the setting and notation required for the propositions. Following original works~\cite{lee2018deep,jacot2018neural} on infinite width networks, we take widths of hidden layers $N_l$ to infinity sequentially, starting from the first layer. Then we consider the outputs of intermediate layer $z^l_{\alpha j}(\mathbf{x})$ at the moment when $N_{l'}=\infty, l'<l$ but $N_l$ is finite, and the respective NTK
\begin{equation}
    \Theta^l_{\alpha\beta,jj'}(\mathbf{x},\mathbf{x}')=\frac{\partial z^l_{\alpha j}(\mathbf{x})}{\partial \widetilde{\mathbf{W}}^{\leq l}} \frac{\partial z^l_{\beta j'}(\mathbf{x}')}{\partial \widetilde{\mathbf{W}}^{\leq l}}.
\end{equation}
Here $\mathbf{W}^{\leq l} = \{\mathbf{w}^{l'}\}_{l'=1}^l \cup \{\mathbf{u}^{l'}\}_{l'=1}^{l-1}$ is the set of all network parameters used to calculate $\mathbf{z}^l(\mathbf{x})$, and the tilde $\;\widetilde{\;}\;$ over the set of parameters denotes that only the trainable parameters are taken from that set. Also note that if we take the last layer $l=L+1$ we will simply get full ensemble NTK: $\Theta^{L+1}_{\alpha\beta,jj'}(\mathbf{x},\mathbf{x}')\equiv\Theta_{\alpha\beta,jj'}(\mathbf{x},\mathbf{x}')$.

Following~\cite{lee2019wide} we also introduce shorthand notations for averages of activation function $\phi(z)$ and its derivative $\dot{\phi}(z)$:
\begin{align}
\label{eq:Phi_def}
    &\Phi(\mathbf{\Sigma}, v_1, v_2) \equiv \langle\phi(v_1 z_1) \phi(v_2 z_2)\rangle, \\
    &\dot{\Phi} (\mathbf{\Sigma}, v_1, v_2) \equiv \langle \dot{\phi}(v_1 z_1) \dot{\phi}(v_2 z_2)\rangle, \\
    &\Phi_z(\mathbf{\Sigma}, v_1, v_2) \equiv \langle z_1 \phi(v_1 z_1) z_2 \phi(v_2 z_2)\rangle.
\end{align}
Here the average is taken w.r.t. random variables $v_1,v_2$ (they may coincide!) and independent of them 2-dimensional Gaussian variable $(z_1,z_2) \sim \mathcal{N}(0,\mathbf{\Sigma})$. The main results of this section are stated in the following two propositions. 

\begin{prop}\label{prop:1}
Consider a hidden layer $l+1\geq 2$ at initialisation. In the sequential limit $N_{l'}\rightarrow \infty, l'\leq l$, the outputs $z^{l+1}_{\alpha j}(\mathbf{x})$ of this layer  converge towards a Gaussian Process with zero mean and the covariance $\langle z^{l+1}_{\alpha j}(\mathbf{x}) z^{l+1}_{\beta j'}(\mathbf{x}')\rangle = \delta_{jj'}\Sigma^{l+1}_{\alpha\beta}(\mathbf{x},\mathbf{x}')$, which can be expressed in terms of covariance of the previous layer outputs $\Sigma^{l}_{\alpha\beta}(\mathbf{x},\mathbf{x}')$ as

\textit{1.} for same models $\alpha=\beta$:
\begin{equation}\label{eq:model_diag_covariance_propagation}
    \Sigma^{l+1}_{\alpha\alpha}(\mathbf{x},\mathbf{x}') = U^{l}_2 \Phi(\mathbf{\Sigma}^{l}_{\alpha\alpha}(\mathbf{x},\mathbf{x}'), v^{l}_\alpha,v^{l}_\alpha);
\end{equation}
\textit{2. for different models $\alpha\ne\beta$:}
\begin{equation}\label{eq:model_offdiag_covariance_propagation}
    \Sigma^{l+1}_{\alpha\beta}(\mathbf{x},\mathbf{x}') = (U^{l}_1)^2\Phi(\mathbf{\Sigma}^{l}_{\alpha\beta}(\mathbf{x},\mathbf{x}'), v^{l}_\alpha,v^{l}_\beta).
\end{equation}
It the first hidden layer the covariance is given by $\langle z^1_{\alpha j}(\mathbf{x}) z^1_{\beta j'}(\mathbf{x}')\rangle = \delta_{jj'}\Sigma^1_{\alpha\beta}(\mathbf{x},\mathbf{x}')$ with $\Sigma^1_{\alpha\beta}(\mathbf{x},\mathbf{x}')=\frac{\mathbf{x}^T\mathbf{x}'}{N_0}$.
\end{prop}
Here the pair $(v^{l}_\alpha,v^{l}_\beta)$ has the same distribution as the initialization distribution of $(v^{l}_{\alpha j},v^{l}_{\beta j})$ for any $j$. $\mathbf{\Sigma}^{l}_{\alpha\beta}(\mathbf{x},\mathbf{x}')$ is a $2\times2$ covariance matrix given by 
\begin{equation}\label{eq:2by2_covariance}
    \mathbf{\Sigma}^l_{\alpha\beta}(\mathbf{x},\mathbf{x}') = \begin{pmatrix}
    \Sigma^l_{\alpha\alpha}(\mathbf{x},\mathbf{x}) & \Sigma^l_{\alpha\beta}(\mathbf{x},\mathbf{x}') \\
    \Sigma^l_{\beta\alpha}(\mathbf{x}',\mathbf{x}) & \Sigma^l_{\beta\beta}(\mathbf{x}',\mathbf{x}')
    \end{pmatrix}.
\end{equation}

\begin{prop}\label{prop:2}
Consider a hidden layer $l+1\geq 2$ at initialisation. In the sequential limit $N_{l'}\rightarrow \infty, l'\leq l$, the NTK of layer $l+1$  converges towards a deterministic value $\Theta^{l+1}_{\alpha\beta,jj'}(\mathbf{x},\mathbf{x}') = \delta_{jj'}\Theta^{l+1}_{\alpha\beta}(\mathbf{x},\mathbf{x}')$ and can be expressed in terms of NTK of the previous layer  $\Theta^{l}_{\alpha\beta}(\mathbf{x},\mathbf{x}')$ and the covariance $\Sigma^{l}_{\alpha\beta}(\mathbf{x},\mathbf{x}')$ as

\textit{1.} for same models $\alpha=\beta$:
\begin{equation}\label{eq:model_diag_NTK_propagation}
\begin{split}
    \Theta^{l+1}_{\alpha\alpha}(\mathbf{x},\mathbf{x}') =& U^l_2 V^l_2\Theta^{l}_{\alpha\alpha}(\mathbf{x},\mathbf{x}') \dot{\Phi}(\mathbf{\Sigma}^{l}_{\alpha\alpha}(\mathbf{x},\mathbf{x}'), v^{l}_\alpha,v^{l}_\alpha)+ U^l_2\Phi(\mathbf{\Sigma}^{l}_{\alpha\alpha}(\mathbf{x},\mathbf{x}'), v^{l}_\alpha,v^{l}_\alpha); \\
    &+T(u^l)\Phi(\mathbf{\Sigma}^{l}_{\alpha\alpha}(\mathbf{x},\mathbf{x}'), v^{l}_\alpha,v^{l}_\alpha) + T(v^l) U^l_2 \Phi_z(\mathbf{\Sigma}^{l}_{\alpha\alpha}(\mathbf{x},\mathbf{x}'), v^{l}_\alpha,v^{l}_\alpha).
\end{split}
\end{equation}
\textit{2. for different models $\alpha\ne\beta$:}
\begin{equation}\label{eq:model_offdiag_NTK_propagation}
    \Theta^{l+1}_{\alpha\beta}(\mathbf{x},\mathbf{x}') = (U^l_1)^2 (V^l_1)^2\Theta^{l}_{\alpha\beta}(\mathbf{x},\mathbf{x}') \dot{\Phi}(\mathbf{\Sigma}^{l}_{\alpha\beta}(\mathbf{x},\mathbf{x}'), v^{l}_\alpha,v^{l}_\beta)+ (U^l_1)^2\Phi(\mathbf{\Sigma}^{l}_{\alpha\beta}(\mathbf{x},\mathbf{x}'), v^{l}_\alpha,v^{l}_\beta)
\end{equation}
In the first hidden layer the NTK is given by $\Theta^1_{\alpha\beta,jj'}(\mathbf{x},\mathbf{x}') = \delta_{jj'}\Theta^1_{\alpha\beta}(\mathbf{x},\mathbf{x}')$ with $\Theta^1_{\alpha\beta}(\mathbf{x},\mathbf{x}')=\Sigma^1_{\alpha\beta}(\mathbf{x},\mathbf{x}')$.
\end{prop}

Let us discuss both propositions in the case of different models $\alpha \ne \beta$. Observe that according to~\eqref{eq:model_offdiag_covariance_propagation} and~\eqref{eq:model_offdiag_NTK_propagation} the model off-diagonal elements $\alpha\ne\beta$ of NTK and covariance are zeroed if the mean of post-activation modulations is zero: $U^l_1=0$.
Moreover, it is sufficient to have only centered post activation modulations $u^L_{\alpha j}$ in the final hidden layer in order to have model diagonal NTK $\Theta^{L+1}_{\alpha\beta}$ and covariance $\Sigma^{L+1}_{\alpha\beta}$ of the ensemble. However, for finite networks we expect that having only last layer modulations $u^L_\alpha$ may be not sufficient for model independence and having modulation $u^l_\alpha$ and $v_\alpha^l$ at intermediate layers  may be beneficial for model independence.

\subsection{Proofs of propositions~\ref{prop:1} and~\ref{prop:2}}\label{sec:proofs}
Both Proposition~\ref{prop:1} and Proposition~\ref{prop:2} are proved by induction with the base of induction corresponding to the first layer $z^1_{\alpha j}$.
\begin{proof}[Proof of Proposition~\ref{prop:1}]
For $l=1$ we note that $z^1_{\alpha j}(\mathbf{x})$ is a linear combination of independent Gaussian random variables $w^1_{ij}$, therefore the result is a Gaussian random variables with the co-variance
\begin{equation}
    \langle z^1_{\alpha j}(\mathbf{x}) z^1_{\beta j'}(\mathbf{x}')\rangle 
    = \frac{1}{N_0}\Big\langle \sum\limits_{i} w^1_{ij} x_i \sum\limits_{i'} w^1_{i'j'} x'_{i'} \Big\rangle = \frac{1}{N_0}\sum\limits_{i, i'} \langle w^1_{ij} w^1_{i'j'}\rangle x_i x'_{i'} = \delta_{jj'}\frac{\mathbf{x}^T\mathbf{x}'}{N_0}.
\end{equation}
In the induction step we assume that $z^l_{\alpha j}(\mathbf{x})$ is a GP with zero mean and covariance $\delta_{jj'} \Sigma_{\alpha \beta}(\mathbf{x},\mathbf{x}')$. Since the network signals $z^l_{\alpha j}$ are Gaussian and co-variance is diagonal w.r.t neuron index $j$, the signals become independent for different $j$ and so is post-activations $x^l_{\alpha j}$. The sum in calculation~\eqref{eq:EE_layer} of $z^{l+1}_{\alpha j}(\mathbf{x})$ then consists of i.i.d random variables. According to CLT, in the infinite width limit $N_l \rightarrow \infty$ any finite collection of $z^{l+1}_{\alpha j}(\mathbf{x})$ (indexed by different $\alpha$,$j$ or $\mathbf{x}$) converge in law to multivariate Gaussian distribution with zero mean and covariance
\begin{equation}
\begin{split}
        \langle z^{l+1}_{\alpha j}(\mathbf{x}) z^{l+1}_{\beta j'}(\mathbf{x}')\rangle ={}& \frac{1}{N_l}\Big\langle \sum\limits_{i}w^{l+1}_{ij} u^l_{\alpha i} \phi(v^l_{\alpha i} z^l_{\alpha i}(\mathbf{x})) \sum\limits_{i'}w^{l+1}_{i'j'} u^l_{\beta i'} \phi(v^l_{\beta i'} z^l_{\beta i'}(\mathbf{x}')) \Big \rangle \\
        ={}& \frac{1}{N_l}\sum\limits_{i, i'} \langle w^{l+1}_{ij} w^{l+1}_{i'j'} \rangle \langle u^l_{\alpha i} u^l_{\beta i'}\rangle \langle \phi(v^l_{\alpha i} z^l_{\alpha i}(\mathbf{x})) \phi(v^l_{\beta i'} z^l_{\beta i'}(\mathbf{x}')) \rangle  \\
        ={}& \delta_{jj'}\langle u^l_{\alpha} u^l_{\beta}\rangle\Phi(\mathbf{\Sigma}^{l}_{\alpha\beta}(\mathbf{x},\mathbf{x}'), v^{l}_\alpha,v^{l}_\beta).
\end{split}
\end{equation}
To get expressions~\eqref{eq:model_diag_covariance_propagation} and~\eqref{eq:model_offdiag_covariance_propagation} we only need to note that $\langle u^l_{\alpha} u^l_{\beta}\rangle=U^l_2, \; \alpha=\beta$ and $\langle u^l_{\alpha} u^l_{\beta}\rangle=(U^l_1)^2, \; \alpha\ne\beta$ and the same for $v^l$.
\end{proof}

\begin{proof}[Proof of Proposition~\ref{prop:2}]
The NTK of the first layer has contribution only from common parameters $w^1_{ij}$
\begin{equation}
    \Theta^{1}_{\alpha\beta, jj'}(\mathbf{x},\mathbf{x}') = \frac{\partial z^1_{\alpha j}(\mathbf{x})}{\partial \mathbf{w}^1} \frac{\partial z^1_{\beta j'}(\mathbf{x}')}{\partial \mathbf{w}^1} = \delta_{jj'} \frac{\mathbf{x}^T \mathbf{x}'}{N_0}.
\end{equation}
The calculation above is the base of induction. Now we proceed with the induction step 
\begin{equation}
\label{eq:NTK_decomp}
\begin{split}
    \Theta^{l+1}_{\alpha\beta, jj'}(\mathbf{x},\mathbf{x}') =& \frac{\partial z^{l+1}_{\alpha j}(\mathbf{x})}{\partial \widetilde{\mathbf{W}}^{\leq l}} \frac{\partial z^{l+1}_{\beta j'}(\mathbf{x}')}{\partial \widetilde{\mathbf{W}}^{\leq l}} + \frac{\partial z^{l+1}_{\alpha j}(\mathbf{x})}{\partial \mathbf{w}^{l+1}} \frac{\partial z^{l+1}_{\beta j'}(\mathbf{x}')}{\partial \mathbf{w}^{l+1}} \\
    &+T(v^l)\frac{\partial z^{l+1}_{\alpha j}(\mathbf{x})}{\partial \mathbf{v}^{l}} \frac{\partial z^{l+1}_{\beta j'}(\mathbf{x}')}{\partial \mathbf{v}^{l}} + T(u^l)\frac{\partial z^{l+1}_{\alpha j}(\mathbf{x})}{\partial \mathbf{u}^{l}} \frac{\partial z^{l+1}_{\beta j'}(\mathbf{x}')}{\partial \mathbf{u}^{l}}.
\end{split}
\end{equation}
Now we examine each term of this decomposition separately. The first term describes propagation of NTK from previous layer to the current one and can be calculated as
\begin{equation}
\label{eq:1st_term}
    \begin{split}
        \frac{\partial z^{l+1}_{\alpha j}(\mathbf{x})}{\partial \widetilde{\mathbf{W}}^{\leq l}} \frac{\partial z^{l+1}_{\beta j'}(\mathbf{x}')}{\partial \widetilde{\mathbf{W}}^{\leq l}} &\stackrel{(1)}={} \sum\limits_{ii'}\frac{\partial z^{l+1}_{\alpha j}(\mathbf{x})}{\partial z^{l}_{\alpha i}(\mathbf{x})} \frac{\partial z^{l+1}_{\beta j'}(\mathbf{x}')}{\partial z^{l}_{\beta i'}(\mathbf{x}')} \frac{\partial z^{l}_{\alpha i}(\mathbf{x})}{\partial \widetilde{\mathbf{W}}^{\leq l}} \frac{\partial z^{l}_{\beta i'}(\mathbf{x}')}{\partial \widetilde{\mathbf{W}}^{\leq l}} \\
        &\stackrel{(2)}={} \frac{1}{N_l}\sum\limits_{ii'}\Theta^{l}_{\alpha\beta, ii'}(\mathbf{x},\mathbf{x}') w^{l+1}_{ij} u^l_{\alpha i} v^l_{\alpha i} \dot{\phi}(z^l_{\alpha i}(\mathbf{x})v^l_{\alpha i}) w^{l+1}_{i'j'} u^l_{\beta i'} v^l_{\beta i'} \dot{\phi}(z^l_{\beta i'}(\mathbf{x'})v^l_{\beta i'}) \\
        &\stackrel{(3)}={} \Theta^{l}_{\alpha\beta, ii}(\mathbf{x},\mathbf{x}') \langle w^{l+1}_{ij} w^{l+1}_{ij'}\rangle \langle u^{l}_{\alpha i} u^{l}_{\beta i}\rangle \langle v^{l}_{\alpha i} v^{l}_{\beta i}\rangle \langle \dot{\phi}(z^l_{\alpha i}(\mathbf{x'})v^l_{\alpha i})\dot{\phi}(z^l_{\beta i}(\mathbf{x'})v^l_{\beta i})\rangle \\
        &\stackrel{(4)}={} \Theta^{l}_{\alpha\beta}(\mathbf{x},\mathbf{x}') \delta_{jj'} \langle u^{l}_{\alpha} u^{l}_{\beta}\rangle \langle v^{l}_{\alpha} v^{l}_{\beta}\rangle \dot{\Phi}(\mathbf{\Sigma}^{l}_{\alpha\beta}(\mathbf{x},\mathbf{x}'), v^{l}_\alpha,v^{l}_\beta).
    \end{split}
\end{equation}
In $(1)$ and $(2)$ we used expression~\eqref{eq:EE_layer} of current post-activations $z^{l+1}$ in terms of previous layer post-activations $z^l$, and then applied chain rule for differentiation. In $(3)$ we used result of Proposition~\ref{prop:1} that $z^l_i$ are i.i.d random variables, which allows to apply the Law of Large Numbers in the limit $N_l \rightarrow \infty$. Finally in $(4)$ we substitute averages and simplified notations.

Now we calculate the next three terms in~\eqref{eq:NTK_decomp}, representing the contribution of current layer parameters to NTK.
The second term
\begin{equation}
\label{eq:2st_term}
    \begin{split}
        \frac{\partial z^{l+1}_{\alpha j}(\mathbf{x})}{\partial \mathbf{w}^{l+1}} \frac{\partial z^{l+1}_{\beta j'}(\mathbf{x}')}{\partial \mathbf{w}^{l+1}} &= \delta_{jj'} \frac{1}{N_l}\sum\limits_i u^l_{\alpha i} \phi(z^l_{\alpha i}(\mathbf{x})v^l_{\alpha i}) u^{l}_{\beta i}\phi(z^l_{\beta i}(\mathbf{x}')v^l_{\beta i}) \\
        &\stackrel{N_l\rightarrow \infty}={} \delta_{jj'}  \langle u^l_{\alpha i} u^l_{\beta i} \rangle \langle \phi(z^l_{\alpha i}(\mathbf{x})v^l_{\alpha i}) \phi(z^l_{\beta i}(\mathbf{x}')v^l_{\beta i})\rangle \\
        &= \delta_{jj'} \langle u^l_\alpha u^l_\beta\rangle \Phi(\mathbf{\Sigma}^{l}_{\alpha\beta}(\mathbf{x},\mathbf{x}'), v^{l}_\alpha,v^{l}_\beta).
    \end{split}
\end{equation}
The third term
\begin{equation}
\label{eq:3st_term}
    \begin{split}
        \frac{\partial z^{l+1}_{\alpha j}(\mathbf{x})}{\partial \mathbf{v}^{l}} \frac{\partial z^{l+1}_{\beta j'}(\mathbf{x}')}{\partial \mathbf{v}^{l}} &= \delta_{\alpha \beta} \frac{1}{N_l}\sum\limits_i w^{l+1}_{ij} u^l_{\alpha i} z^l_{\alpha i}(\mathbf{x}) \dot{\phi}(z^l_{\alpha i}(\mathbf{x})v^l_{\alpha i}) w^{l+1}_{ij'} u^l_{\beta i} z^l_{\beta i}(\mathbf{x}') \dot{\phi}(z^l_{\beta i}(\mathbf{x}')v^l_{\beta i}) \\
        &\stackrel{N_l\rightarrow \infty}={} \delta_{\alpha \beta} \langle w^{l+1}_{ij} w^{l+1}_{ij'}\rangle \langle u^l_{\alpha i} u^l_{\beta i} \rangle \langle z^l_{\alpha i}(\mathbf{x}) \dot{\phi}(z^l_{\alpha i}(\mathbf{x})v^l_{\alpha i})z^l_{\beta i}(\mathbf{x}') \dot{\phi}(z^l_{\beta i}(\mathbf{x}')v^l_{\beta i})\rangle \\
        &= \delta_{\alpha\beta} \delta_{jj'} u^l_2 \Phi_z(\mathbf{\Sigma}^{l}_{\alpha\beta}(\mathbf{x},\mathbf{x}'), v^{l}_\alpha,v^{l}_\beta).
    \end{split}
\end{equation}
The fourth term
\begin{equation}
\label{4st_term}
    \begin{split}
        \frac{\partial z^{l+1}_{\alpha j}(\mathbf{x})}{\partial \mathbf{u}^{l}} \frac{\partial z^{l+1}_{\beta j'}(\mathbf{x}')}{\partial \mathbf{u}^{l}} &= \delta_{\alpha \beta} \frac{1}{N_l}\sum\limits_i w^{l+1}_{ij} \phi(z^l_{\alpha i}(\mathbf{x})v^l_{\alpha i}) w^{l+1}_{ij'}\phi(z^l_{\beta i}(\mathbf{x}')v^l_{\beta i}) \\
        &\stackrel{N_l\rightarrow \infty}={} \delta_{\alpha \beta} \langle w^{l+1}_{ij} w^{l+1}_{ij'}\rangle \langle \phi(z^l_{\alpha i}(\mathbf{x})v^l_{\alpha i}) \phi(z^l_{\beta i}(\mathbf{x}')v^l_{\beta i})\rangle \\
        &= \delta_{\alpha\beta} \delta_{jj'} \Phi(\mathbf{\Sigma}^{l}_{\alpha\beta}(\mathbf{x},\mathbf{x}'), v^{l}_\alpha,v^{l}_\beta).
    \end{split}
\end{equation}
For each of last three terms the steps are analogous to steps (2-4) in the calculation of the first term~\eqref{eq:1st_term}. Finally, to get the expressions~\eqref{eq:model_diag_NTK_propagation} and~\eqref{eq:model_offdiag_NTK_propagation} of Proposition~\ref{prop:2} we again note that $\langle u^l_{\alpha} u^l_{\beta}\rangle=U^l_2, \; \alpha=\beta$ and $\langle u^l_{\alpha} u^l_{\beta}\rangle=(U^l_1)^2, \; \alpha\ne\beta$ and the same for $v^l$.
\end{proof}

\subsection{Proofs of Theorems~\ref{theorem_1} and~\ref{theorem_2}}\label{sec:thm_proofs}

In the proofs of Theorems~\ref{theorem_1} and~\ref{theorem_2} we rely on, respectively, Propositions~\ref{prop:1} and~\ref{prop:2}.

\begin{proof}[Proof of Theorem~\ref{theorem_1}]\hfill

1. \emph{(Gaussianity)} Convergence to a zero mean Gaussian process is already proved in Proposition~\ref{prop:1}. 

2. \emph{(Independence)}
To show that covariance of different $\alpha\ne\beta$ model outputs $\mathbb{E}[f_\alpha(\mathbf{x})f_\beta(\mathbf{x}')]$ vanish when $U_1^L=0$ we recall that $\mathbb{E}[f_\alpha(\mathbf{x})f_\beta(\mathbf{x}')] = \Sigma^{L+1}_{\alpha\beta}(\mathbf{x},\mathbf{x}')$, which can be recursively calculated from~\eqref{eq:model_offdiag_covariance_propagation} starting with $\Sigma^{1}_{\alpha\beta}(\mathbf{x},\mathbf{x}')$. From~\eqref{eq:model_offdiag_covariance_propagation} we see that regardless of $\Sigma^{L}_{\alpha\beta}(\mathbf{x},\mathbf{x}')$, the final layer covariance $\Sigma^{L+1}_{\alpha\beta}(\mathbf{x},\mathbf{x}')$ will be zero due to its proportionality to $(U_1^L)^2$. 

3. \emph{(Breakdown of independence)} 
Let $U_1^L\ne 0$ and suppose that for different ensemble models $\alpha\ne\beta$ and at some non-zero inputs $\mathbf{x},\mathbf{x'}\ne 0$ the output covariance $\Sigma^{L+1}_{\alpha\beta}(\mathbf{x},\mathbf{x}')=0$. We need to show that this is only possible if $L=1$ and $\mathbf x,\mathbf x'$ are linearly dependent. 

First note that since, by~\eqref{eq:model_offdiag_covariance_propagation}, $\Sigma^{L+1}_{\alpha\beta}(\mathbf{x},\mathbf{x}')=(U^{L}_1)^2\Phi(\mathbf{\Sigma}^{L}_{\alpha\beta}(\mathbf{x},\mathbf{x}'), v^{L}_\alpha,v^{L}_\beta)$, the factor $\Phi(\mathbf{\Sigma}^{L}_{\alpha\beta}(\mathbf{x},\mathbf{x}'), v^{L}_\alpha,v^{L}_\beta)=\mathbb E(\phi(v^{L}_\alpha z_1) \phi(v^{L}_\beta z_2))$ has to vanish. This is clearly impossible if $\phi(z)>0$ for all $z\in\mathbb{R}$. Since by assumption $\phi$ is a non-decreasing nonnegative function, we are left to consider the case $\{z| \; \phi(z)=0\}=(-\infty,a] \ne \varnothing$.

Express  $\Phi(\mathbf{\Sigma}^{L}_{\alpha\beta}(\mathbf{x},\mathbf{x}'), v^{L}_\alpha,v^{L}_\beta)$ in terms of the expectation of $\phi(v^{L}_\alpha z_1) \phi(v^{L}_\beta z_2)$ conditioned on fixed $v^{L}_\alpha,v^{L}_\beta$: 
\begin{equation}\label{eq:aslab}
    \Phi(\mathbf{\Sigma}^{L}_{\alpha\beta}(\mathbf{x},\mathbf{x}'), v^{L}_\alpha,v^{L}_\beta)=\mathbb E(\phi(v^{L}_\alpha z_1) \phi(v^{L}_\beta z_2)) = \mathbb E_{v^{L}_\alpha,v^{L}_\beta}[\mathbb E_{z_1,z_2}(\phi(v^{L}_\alpha z_1) \phi(v^{L}_\beta z_2)|v^{L}_\alpha,v^{L}_\beta)].
\end{equation}
For each $v^{L}_\alpha,v^{L}_\beta$,  the conditioned 2D random vector $\mathbf y_{v^{L}_\alpha, v^{L}_\beta}=(v^{L}_\alpha z_1, v^{L}_\beta z_2)|v^{L}_\alpha,v^{L}_\beta$ is normal, and the conditional expectation in~\eqref{eq:aslab} can be written as
\begin{equation}
    \mathbb E_{z_1,z_2}(\phi(v^{L}_\alpha z_1) \phi(v^{L}_\beta z_2)|v^{L}_\alpha,v^{L}_\beta) = \mathbb E (F(\mathbf y_{v^{L}_\alpha, v^{L}_\beta})),
\end{equation}
where $F(y_1,y_2)=\phi(y_1)\phi(y_2)$ is a nonnegative function which is strictly positive on the quadrant $R_a=(a,+\infty)\times (a,+\infty).$

Since $\Phi(\mathbf{\Sigma}^{L}_{\alpha\beta}(\mathbf{x},\mathbf{x}'), v^{L}_\alpha,v^{L}_\beta)=0$ and $F\ge 0$, we must have $\mathbb E (F(\mathbf y_{v^{L}_\alpha, v^{L}_\beta}))=0$ almost surely in $v^{L}_\alpha, v^{L}_\beta$. Moreover, since $F$ is strictly positive on the quadrant $R_a$, the distributions of $\mathbf y_{v^{L}_\alpha, v^{L}_\beta}$ must be degenerate and supported outside the quadrant $R_a$ almost surely in $v^{L}_\alpha, v^{L}_\beta$. Therefore, since $\mathbf y_{v^{L}_\alpha, v^{L}_\beta}$ is normal, for (a.s.) any $v^{L}_\alpha, v^{L}_\beta$: 
\begin{enumerate}
    \item either at least one component of $\mathbf y_{v^{L}_\alpha, v^{L}_\beta}$ is 0 a.s.,
    \item or the components of $\mathbf y_{v^{L}_\alpha, v^{L}_\beta}$ are linearly dependent and negatively correlated.
\end{enumerate}   
The variables $v^{L}_\alpha, v^{L}_\beta$ have the same distribution and, by assumption, do not vanish identically. Then, in both cases $\alpha=\beta$ (when they coincide) and $\alpha\ne \beta$, the event $Q=\{v^{L}_\alpha v^{L}_\beta>0\}$ has a positive probability.  We will show now that conditioned on this event, only case 2 above is possible, and only if $L=1$ and $\mathbf x,\mathbf x'$ are linearly dependent.

\emph{Case 1.} To exclude case 1, take, for example, the first component $y^{(1)}_{v^{L}_\alpha, v^{L}_\beta}$ of $\mathbf y_{v^{L}_\alpha, v^{L}_\beta}$ and observe that, by~\eqref{eq:2by2_covariance}, 
\begin{equation}\label{eq:ey1}
    \mathbb E\big[(y^{(1)}_{v^{L}_\alpha, v^{L}_\beta})^2\big] =
    \mathbb E\big[(v^{L}_\alpha z_1)^2|v^{L}_\alpha\big] =(v^{L}_\alpha)^2\big(\mathbf\Sigma^L_{\alpha\beta}(\mathbf x,\mathbf x')\big)_{11}=
    (v^{L}_\alpha)^2\Sigma^L_{\alpha\alpha}(\mathbf x,\mathbf x).
\end{equation}
If $L>1$, then, by Proposition~\ref{prop:1},
\begin{equation}
    \Sigma^L_{\alpha\alpha}(\mathbf x,\mathbf x) = U_2^{L-1}\Phi(\mathbf\Sigma_{\mathbf x,\mathbf x}^{L-1},v^{L-1}_\alpha,v^{L-1}_\alpha) = U_2^{L-1}\mathbb E_{z\sim\mathcal{N}(0,\Sigma^{L-1}_{\alpha\alpha}(\mathbf{x},\mathbf{x})), v_\alpha^{L-1}}\big[\phi(v^{L-1}_\alpha z)^2\big]>0
\end{equation}
since $U_2^{L-1}>0$ and $v_\alpha^{L-1}$ is not 0 a.s. But then, since $v_\alpha^{L}$ is not 0 a.s. as well, $\mathbb E\big[(y^{(1)}_{v^{L}_\alpha, v^{L}_\beta})^2\big]>0$, i.e. $y^{(1)}_{v^{L}_\alpha, v^{L}_\beta}$ does not vanish.

On the other hand, if $L=1$, then, by Proposition~\ref{prop:1}, $\Sigma^1_{\alpha\alpha}(\mathbf{x},\mathbf{x})=\frac{\|\mathbf{x}\|^2}{N_0}>0$ for nonzero $\mathbf x$, again meaning that $y^{(1)}_{v^{L}_\alpha, v^{L}_\beta}$ does not vanish.

\emph{Case 2.} The negative correlation of the components of $\mathbf y_{v^{L}_\alpha, v^{L}_\beta}$ means that 
\begin{equation}\label{eq:ey2}
    0>\mathbb E\big[y^{(1)}_{v^{L}_\alpha, v^{L}_\beta}y^{(2)}_{v^{L}_\alpha, v^{L}_\beta}\big] =
    \mathbb E\big[v^{L}_\alpha z_1v^{L}_\beta z_2|v^{L}_\alpha,v^{L}_\beta\big] =v^{L}_\alpha v^L_\beta\big(\mathbf\Sigma^L_{\alpha\beta}(\mathbf x,\mathbf x')\big)_{12}=
    v^{L}_\alpha v^{L}_\beta\Sigma^L_{\alpha\beta}(\mathbf x,\mathbf x').
\end{equation}
If $L>1$, then, by Proposition~\ref{prop:1},
\begin{equation}\label{eq:alabm}
    \Sigma^L_{\alpha\beta}(\mathbf x,\mathbf x') = U_2^{L-1}\Phi(\mathbf\Sigma_{\mathbf x,\mathbf x'}^{L-1},v^{L-1}_\alpha,v^{L-1}_\beta) = U_2^{L-1}\mathbb E_{z\sim\mathcal{N}(0,\Sigma^{L-1}_{\alpha\beta}(\mathbf{x},\mathbf{x}')), v_\alpha^{L-1}, v_\beta^{L-1}}\big[\phi(v^{L-1}_\alpha z_1)\phi(v^{L-1}_\beta z_2)\big]\ge 0
\end{equation}
since $\phi\ge 0$. Since we are considering the event $Q=\{v^{L}_\alpha v^{L}_\beta>0\},$ inequality~\eqref{eq:alabm} contradicts inequality~\eqref{eq:ey2}.

On the other hand, if $L=1$, then, again by Proposition~\ref{prop:1},
$\Sigma^1_{\alpha\beta}(\mathbf{x},\mathbf{x}')=\frac{\mathbf{x}^T\mathbf{x}'}{N_0}$. This expression can indeed be negative for certain combinations of $\mathbf x,\mathbf x'$.  Recall that we also require the components of $\mathbf y_{v^{L}_\alpha, v^{L}_\beta}$ to be linearly dependent. The full covariance matrix of $\mathbf y_{v^{L}_\alpha, v^{L}_\beta}$ is  
\begin{align}
    \Sigma_{\mathbf y_{v^{L}_\alpha, v^{L}_\beta}} ={}& \operatorname{diag}(v_\alpha^1,v_\beta^1)
    \begin{pmatrix}
    \Sigma_{\alpha\alpha}^1(\mathbf x,\mathbf x) & \Sigma_{\alpha\beta}^1(\mathbf x,\mathbf x')\\
    \Sigma_{\beta\alpha}^1(\mathbf x',\mathbf x) & \Sigma_{\beta\beta}^1(\mathbf x',\mathbf x')
    \end{pmatrix}
    \operatorname{diag}(v_\alpha^1,v_\beta^1)\\
    ={}&\frac{1}{N_0}\operatorname{diag}(v_\alpha^1,v_\beta^1)
    \begin{pmatrix}
    \mathbf x^T\mathbf x & \mathbf x^T\mathbf x'\\
    (\mathbf x')^T\mathbf x & (\mathbf x')^T\mathbf x'
    \end{pmatrix}
    \operatorname{diag}(v_\alpha^1,v_\beta^1).
\end{align}
Under condition $Q$ and nonzero $\mathbf x,\mathbf x'$, this matrix has rank 1 if and only if $\mathbf x$ and $\mathbf x'$ are linearly dependent. This completes the proof.

\end{proof}

\begin{proof}[Proof of Theorem~\ref{theorem_2}]\hfill

1. \emph{(Determinacy)} Convergence of the NTK to a deterministic limit is already proved in Proposition~\ref{prop:2}. 

2. \emph{(Dynamic independence)} With $U_1^L=0$, the vanishing of off-diagonal $\alpha\ne\beta$ components of ensemble NTK $\Theta_{\alpha\beta}(\mathbf{x}, \mathbf{x}')\equiv \Theta^{L+1}_{\alpha\beta}(\mathbf{x}, \mathbf{x}')$ follows directly from~\eqref{eq:model_offdiag_NTK_propagation} at $l=L$.

3. \emph{(Breakdown of dynamic independence)} By iteratively applying~\eqref{eq:model_offdiag_NTK_propagation}, off-diagonal components of the ensemble NTK  are given by
\begin{equation}\label{eq:last_layer_NTK_decomp}
    \Theta_{\alpha\beta}(\mathbf{x}, \mathbf{x}') = \sum\limits_{l=1}^{L} \left(\prod\limits_{m=l}^{L} (U_1^m)^2 (V_1^m)^2 \dot{\Phi}(\mathbf{\Sigma}^{m}_{\alpha\beta}(\mathbf{x},\mathbf{x}'), v^{m}_\alpha,v^{m}_\beta)\right) \Sigma^l_{\alpha\beta}(\mathbf{x}, \mathbf{x}') + \Sigma^{L+1}_{\alpha\beta}(\mathbf{x}, \mathbf{x}').
\end{equation}
Note that for $\phi\in\mathcal{S}$ we automatically have $\Sigma^l_{\alpha\beta}(\mathbf{x}, \mathbf{x}') \geq 0$ for all $l>1$ and $\dot{\Phi}(\mathbf{\Sigma}^{l}_{\alpha\beta}(\mathbf{x},\mathbf{x}'), v^{l}_\alpha,v^{l}_\beta)\geq 0$ for $l\geq 1$. If we additionally assume $\mathbf{x}^T\mathbf{x}' \geq 0$ we get $\Sigma^1_{\alpha\beta}(\mathbf{x}, \mathbf{x}') \geq 0$ and thus all the terms of the sum in~\eqref{eq:last_layer_NTK_decomp} are non-negative. Then for $\Theta_{\alpha\beta}(\mathbf{x}, \mathbf{x}')>0$ it is sufficient to have $\Sigma^{L+1}_{\alpha\beta}(\mathbf{x}, \mathbf{x}')>0$, which we have already proved in Theorem~\ref{theorem_1} under a weaker condition on inputs: $\mathbf{x},\mathbf{x'}\ne0$ and $\mathbf{x}\ne -c\mathbf{x}', c>0$. 
\end{proof}

\subsection{Stationarity of NTK}\label{sec:ntk_stat_proof}
The full proof of the property (A3) for infinitely wide neural networks is more involved than the proofs of properties (A1, A2). At the same time, this property is qualitative in contrast to properties (A1, A2), which have the quantitative component -- recursive formulas for calculation of NTK and covariance. For these reasons in the present work we only present a non-rigorous argument showing that the extension of (A3) to Embedded Ensembles, namely the constancy of EE NTK $\Theta_{\alpha\beta}(\mathbf{x},\mathbf{x}')$ during GD training, also holds.

We base our argument on the approach of~\cite{Dyer2020AsymptoticsOW}, which we shortly describe here. For the details the reader should refer to original work~\cite{Dyer2020AsymptoticsOW}. The average of NTK at some time $t$ w.r.t. initialization is considered on the level of formal Taylor expansion at $t=0$
\begin{equation}\label{eq:NTK_Taylor_expansion}
    \langle \Theta(\mathbf{x},\mathbf{x}', t) \rangle =  \sum\limits_{n=0}^\infty \frac{t^n}{n!}\left\langle \left(\frac{d}{dt}\right)^n\Theta(\mathbf{x},\mathbf{x}', t)\Big|_{t=0}\right\rangle.
\end{equation}
Each term in this sum is given by a (collection of) certain \textit{correlation function} - the average at initialisation of products of network outputs and/or their derivatives. For example, in case of MSE loss the term $n=1$ in~\eqref{eq:NTK_Taylor_expansion} is given by
\begin{equation}
    \left\langle \frac{d}{dt}\Theta(\mathbf{x},\mathbf{x}', t)\Big|_{t=0}\right\rangle = \frac{1}{B}\sum\limits_{b=1}^B \left\langle\frac{\partial f(\mathbf{x})}{\partial \mathbf{W}}\frac{\partial^2 f(\mathbf{x}')}{\partial \mathbf{W}\partial \mathbf{W}'}\frac{\partial f(\mathbf{x}_b)}{\partial \mathbf{W}'}f(\mathbf{x}_b)\right\rangle + \mathbf{x}\leftrightarrow \mathbf{x}'.
\end{equation}
Then it is shown that this correlation function scales with the network width $N$ as $\langle \frac{d}{dt}\Theta(\mathbf{x},\mathbf{x}', t)\big|_{t=0}\rangle = O(\frac{1}{N})$, and also for any correlation function $\langle F(\overline{\mathbf{x}})\rangle = O(N^s)$ its time derivative at initialisation have the same scaling $\langle \frac{d}{dt}F(\overline{\mathbf{x}})\big|_{t=0}\rangle = O(N^s)$. Thus for NTK we have $\langle \Theta(\mathbf{x},\mathbf{x}', t) \rangle = \langle \Theta(\mathbf{x},\mathbf{x}', t=0) \rangle + O(\frac{1}{N})$.

This results are based on the conjecture which relates upper bound on width scaling of correlation function with the number of connected components in the cluster graph associated with the given correlation function. The intuition behind this conjecture comes from considering the case of trivial network activation function $\phi(z)=z$, which makes the network linear
\begin{equation}\label{eq:lin_network}
    f(\mathbf{x}) = N^{-\frac{L}{2}}\sum\limits_{i_0,i_1,\ldots,i_L}V_{i_L}W^L_{i_Li_{L-1}}\ldots W^2_{i_2i_1} U_{i_1i_0}x_{i_0}.
\end{equation}
In this section we show that if a certain correlation function composed of network outputs~\eqref{eq:lin_network} has width scaling $O(N^s)$, its Embedded Ensemble version with each $f(\mathbf{x})$ replaced with $f_\alpha(\mathbf{x})$ given by 
\begin{equation}\label{eq:EE_lin_network}
    f_\alpha(\mathbf{x}) = N^{-\frac{L}{2}}\sum\limits_{i_0,i_1,\ldots,i_L}V_{i_L}u^L_{\alpha i_l}W^L_{i_Li_{L-1}}\ldots u^2_{\alpha i_2}W^2_{i_2i_1} u^1_{\alpha i_1}U_{i_1i_0}x_{i_0},
\end{equation}
has the same width scaling $O(N^s)$ for any choice model indices $\alpha$. 

For the linear network functions~\eqref{eq:lin_network} all the correlation functions are some moments of multivariate Gaussian distribution of model parameters at initialization, which according to \textit{Isserlis theorem} reduce to the sum over all possible pairings of identical model parameters in the correlation function. These pairing produce delta functions $\delta_{i_l i_{l'}'}$ between some pairs of indices thus effectively reducing the number of sums in the corresponding contribution to correlation function. The number of sums over hidden layer indices $i_l$ left unconstrained together with normalization factors $N^{-\frac{L}{2}}$ define the width scaling of corresponding contribution to the correlation function. Finally, the contribution with the highest width scaling define the width scaling of the whole correlation function. 

Now consider correlation function composed of EE outputs~\eqref{eq:EE_lin_network} with arbitrary model indices $\alpha$. We can calculate corresponding average in two stages: first average over all parameters $W^l, U, V$, and then over all modulations $u^l_\alpha$. The first averaging will produce the same set of pairings between identical parameters $W^l_{ij}$ as for the correlation function with outputs~\eqref{eq:lin_network} with the same width scaling for each pairing. The average over modulations will result in the multiplicative factor which will depend on the chosen set of model indices $\alpha$ and on the chosen $W^l$ pairing, but will not depend on the network width $N$ and thus being $O(N^0)$ (however, this average may introduce additional delta functions which will lower the width scaling exponent $s$). Thus we can map contributions to single network correlation function to contributions to EE correlation function conserving the upper bound on the width scaling and thus the upper scaling of the whole correlation functions will also be conserved.     

Finally, the time derivative of Embedded Ensemble NTK $\Theta_{\alpha\beta}(\mathbf{x},\mathbf{x}')$ is also represented as combination of correlation functions considered above. For example, in case of MSE we have
\begin{equation}
    \left\langle \frac{d}{dt}\Theta_{\alpha\beta}(\mathbf{x},\mathbf{x}', t)\Big|_{t=0}\right\rangle = \frac{1}{B}\sum\limits_{b=1}^B\sum\limits_{\gamma=1}^M \left\langle\frac{\partial f_\alpha(\mathbf{x})}{\partial \mathbf{W}}\frac{\partial^2 f_\beta(\mathbf{x}')}{\partial \mathbf{W}\partial \mathbf{W}'}\frac{\partial f_\gamma(\mathbf{x}_b)}{\partial \mathbf{W}'}f_\gamma(\mathbf{x}_b)\right\rangle + (\alpha,\mathbf{x})\leftrightarrow (\beta,\mathbf{x}').
\end{equation}
Thus we have shown that NTK of Embedded Ensemble stays constant during training and its deviation from initial value is bounded as $O(N^{-1})$. However, note that this upper bound is not uniform w.r.t. time $t$ and large times may require large network widths.

\section{Experiment details}\label{sec:experiment_details}
  We performed experiments on CIFAR100 using a simple 4-hidden-layer convolutional architecture  similar to VGG16~\cite{simonyan2014very}, except for a fewer number of layers: the filters have receptive fields 3x3; the number of filters in the hidden layers is 64 -- 128 -- 256 -- 512.
  We initialized convolutional layers with Xavier normal  and initialized modulations with standard normal distribution. 
  Our architectures do not include batch normalization and bias term, both in fully-connected and in convolutional layers. We trained our models with cross--entropy loss and SGD with momentum equal to 0.9,  and regularized the weights with l2-norm. We trained all models with batch size 128 for 500 or 1000 epochs. We used a baseline learning rate equal to 0.01 for the first 500 epochs and then started to decrease it linearly to 1e-4 at the end of the training.  We used standard CIFAR100 train/test split, and extracted randomly 10k examples from train part for validation. 
  
  We used the standard parameterization~\cite{sohl2020infinite}. We tried standard and NTK-parametrization and found their accuracy to be similar, but convergence faster in the standard case. We remark that both parameterizations yield a kernel model in the infinite width limit. 
  
  In our experiments, we do not modulate the inputs of the first parameter layer and outputs of the last parameter layer (in contrast to the original BatchEnsemble architecture~\cite{wen2019batchensemble}). We have observed this to noticeably improve the overall accuracy of resulting models (compare \textit{blue} and \textit{yellow} curves on Figure~\ref{fig:ablation}) and remove unstable training behaviours. This can be explained as follows. The input layer has only 3 RGB channels. Due to randomness in the modulations, input modulations in some models can be much smaller than in others; accordingly, the gradients of some models can be dominated by gradients of other models, creating a significant imbalance in the ensemble. Such a disparity between different models does not occur if we modulate wide hidden layers, thanks to the averaging occurring by the law of large numbers. We indeed observe a much smaller spread of gradient magnitudes among the models if only hidden layers are modulated, see Figure~\ref{fig:Gradients_at_init} (left). We also exclude modulations in the output layer since each of the 100 output channels represents a particular target class and their randomization by modulations is pointless.   

\section{Additional experiments}\label{sec:add_exp}

\begin{figure*}[tb]
    \centering
    \setlength{\fboxsep}{0pt}
    \includegraphics[scale=0.42, trim=3mm 3mm 3mm 3mm, clip]{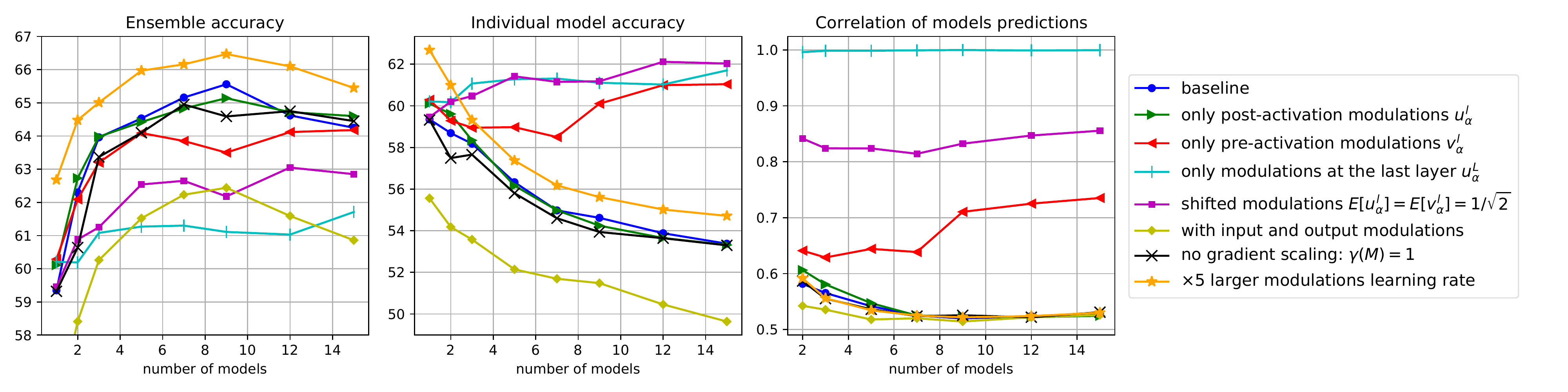}
    \caption{An ablation study for the baseline BatchEnsemble. The \textit{green} and \textit{red} curves use only a part of baseline (\textit{blue}) modulations. The \textit{cyan} curve uses only post-activation modulations at the last layer. The \textit{purple} curve uses shifted modulations with both variance and expectation equal to $1/\sqrt{2}$. The \textit{yellow} curve adds modulations to model inputs and outputs. The \textit{black} curve uses $1/M$ times smaller learning rate for common parameters, thus effectively using unscaled gradients. The \textit{orange} curve uses $\times5$ larger learning rate for all modulations.}
    \label{fig:ablation}
\end{figure*}

\begin{figure}[tb]
 \centering
 \includegraphics[scale=0.55, trim=3mm 4mm 3mm 3mm, clip]{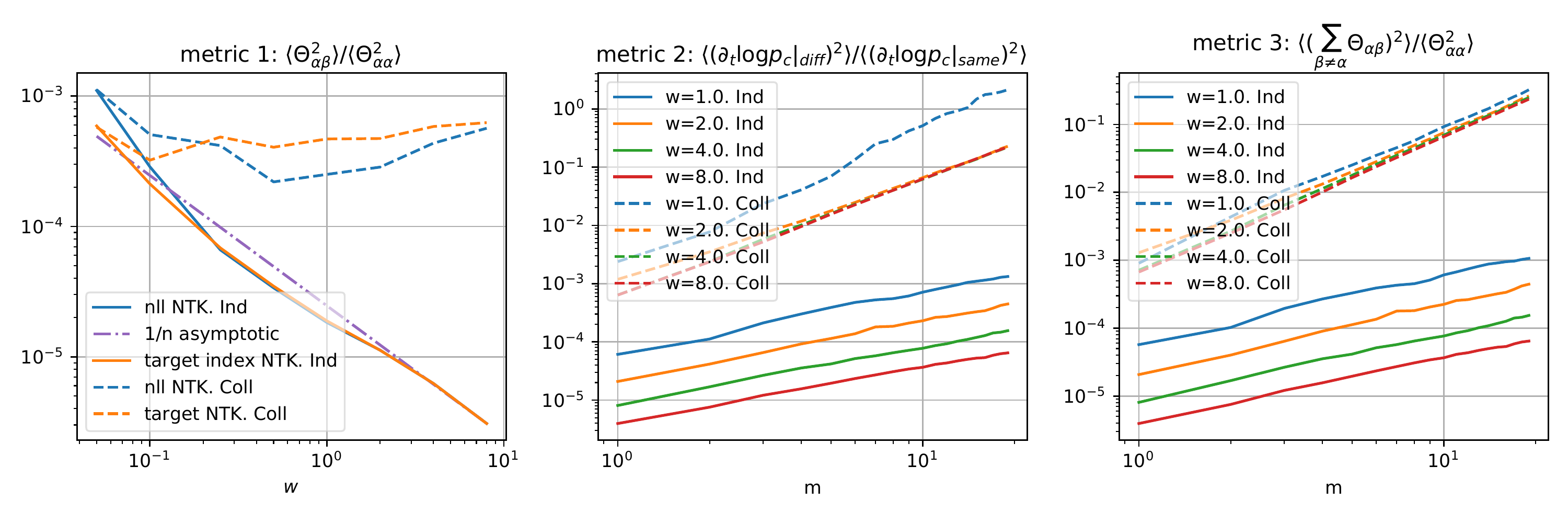}
  \caption{Each of the three plots gives the respective metric for intter-model interaction as described in the text. "Ind" denotes EE in the individual regime and "Coll" denotes EE in collective regime. In the left figure we can see that off-diagonal elements of NTK decay with width as $\sim w^{-1}$ in individual regime and stay constant in collective regime. In the central and right figures we can see that respective metrics scale with number of models $m$ and network width $w$ as $\sim\frac{m}{w}$ for individual regime and $\sim m^2$ for collective regime.}
  \label{fig:NTK_stat_at_init}
 \end{figure}
 
 \begin{figure}[tb]
 \centering
 \includegraphics[scale=0.65, trim=3mm 4mm 3mm 3mm, clip]{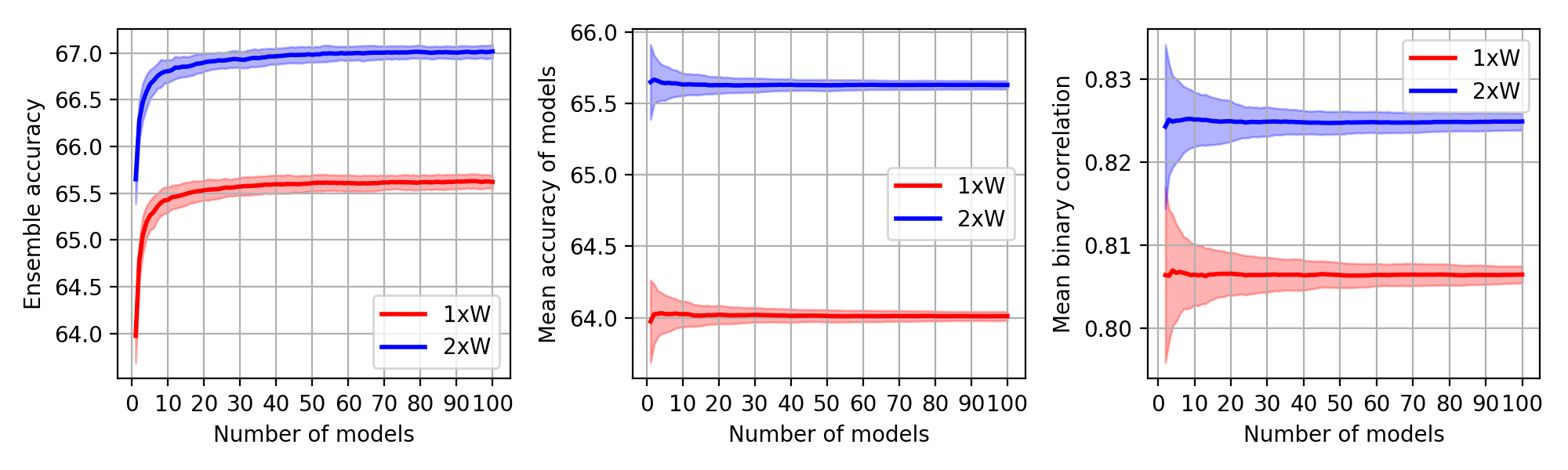}
  \caption{Shifted Gaussian MC Dropout with masks sampled independently for each forward pass during training. Number of models corresponds to the number of masks used during inference. The models are in collective regime, with no clear optimum for number of models, and no decrease in individual model accuracy with increase of ensemble size.}
  \label{fig:gaussian_mc_dropout_shifted}
 \end{figure}
 
  \begin{figure}[tb]
 \centering
 \includegraphics[scale=0.65, trim=3mm 4mm 3mm 3mm, clip]{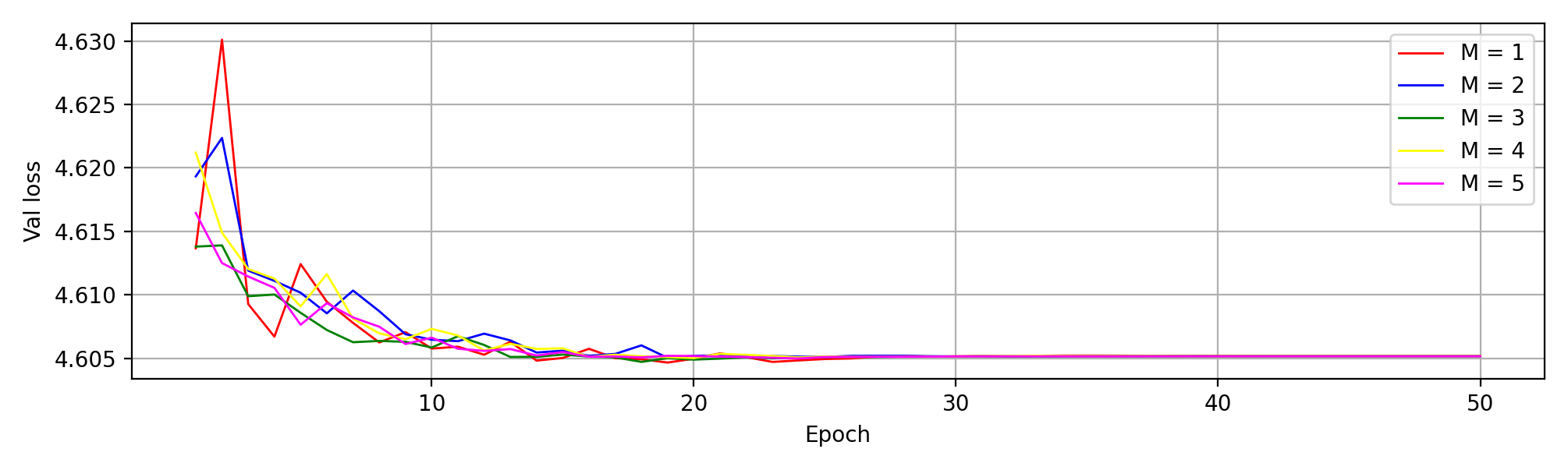}
\caption{Centered Gaussian MC Dropout with independent mask sampling during training. Due to the random mask being applied on each forward pass during training, and the fact that mask weights are centered around zero, there is no gradient direction beneficial to most batches, and model doesn't learn.}
  \label{fig:gaussian_mc_dropout_no_shift}
 \end{figure}
 
 \begin{figure}[tb]
 \centering
 \includegraphics[scale=0.4, trim=30mm 4mm 3mm 3mm, clip]{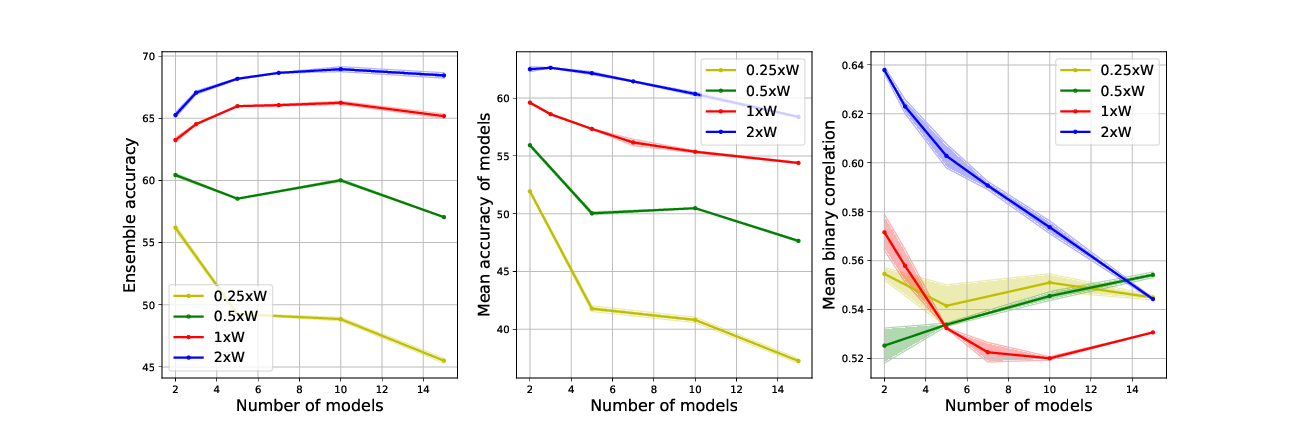}
\caption{Centered Gaussian MC Dropout. Number of models correspond to the fixed number of masks we use during training and inference. In accordance with the theory the models are in individual regime with low correlation between models and optimal finite ratio of number of models to the width.}
  \label{fig:MC-Dropout-zero-centred}
 \end{figure}
 
\subsection{Ablation study}
In Figure~\ref{fig:ablation} we study the effects of different changes in architecture or training procedure of our baseline BatchEnsemble. According to the theory, the presence of only pre-activation modulations $v^l_\alpha$ (\textit{red} curve) or initializing modulation with non-zero mean (\textit{purple} curve) lead to non-diagonal NTK and therefore more dependent models. We observe that this leads to larger correlations (on the right plot) and smaller ensemble accuracy (on the left plot). Again, according to the theory the presence of only post-activation modulations $u^l_\alpha$ (\textit{green} curve) is sufficient for dynamic independence of ensemble models. We see that in experiment the ensemble accuracy and model correlations are roughly the same as for the baseline. The \textit{cyan} curve has $\approx 100\%$ correlations and low ensemble accuracy. The behavior in the latter case is explained by the linear dependence of the output on the modulations in this case: the models in the ensemble are the same up to the last parameter layer, which makes individual modulations $u^L_\alpha$ train on the same convex loss surface and eventually converge to the same unique global minimum. This phenomenon does not allow to directly check predictions of Theorems~\ref{theorem_1} and~\ref{theorem_2} about special role in last layer modulations. However, this difficulty can be overcome by freezing last layer modulations which leads to LLD (last layer dropout) ensemble.  
The \textit{yellow} curve shows that the presence of input and output modulations leads to a severe drop in performance, as explained earlier.
The absence of common parameters gradient scaling (\textit{black} curve) just slightly decreases the ensemble accuracy. This seeming disagreement with the theory can be explained by the number of epochs (1000) being several times larger than needed to train the baseline ensemble, thus leaving enough time to train the models even with very low learning rate. The \textit{orange} curve is motivated by our observation that modulations gradients turned out to be significantly lower than common weights gradients, leading to very slow training of modulations. We can see that compensation of this effect with an increase of modulations learning rate leads to improvement of ensemble accuracy by $\sim 1\%$.    
 
\subsection{Model dependence at initialization}\label{sec:NTK_init_and_dyn}
In this section we provide experiments for various NTK-based metrics of inter-model interaction at initialization. We focus on BatchEnsemble with all modulations based on convolutional network with 4 hidden layers described in the main text and trained on CIFAR100 dataset. 

The full NTK for this problem, calculated on batch of size $B$ and for ensemble of size $M$ is a six dimensional tensor with dimensions $(C,C,M,M,B,B)$, which is quite large for CIFAR100 with $C=100$. For these reason we calculate not the full NTK but its two class reduced variants with shape $(M,M,B,B)$ inspired by discussion of NTK for cross entropy losses at the end of Section~\ref{sec:NTK_background}. We refer to this variants as "nll" NTK and "target" NTK which are defined as
\begin{align}
    &\Theta^{\text{nll}}_{\alpha\beta}(\mathbf{x}_a,\mathbf{x}_b) =  \sum\limits_{i,j=1}^C\frac{\partial L(\mathbf{f}_\alpha(\mathbf{x}_a), c_a)}{\partial f_{i\alpha}(\mathbf{x}_a)} \Theta_{ij,\alpha\beta}(\mathbf{x}_a, \mathbf{x}_b) \frac{\partial L(\mathbf{f}_\beta(\mathbf{x}_b), c_b)}{\partial f_{j\beta}(\mathbf{x}_b)}, \\
    &\Theta^{\text{target}}_{\alpha\beta}(\mathbf{x}_a,\mathbf{x}_b) = \Theta_{c_ac_b,\alpha\beta}(\mathbf{x}_a,\mathbf{x}_b).
\end{align}
Here $c_a,c_b$ denotes the class index of $\mathbf{x}_a,\mathbf{x}_b$, and $L(\mathbf{f},c)$ is the cross entropy loss~\eqref{eq:cross_entropy}. The first metric measuring cross model interaction is the average $\Theta_{\alpha\beta}^2(\mathbf{x}_a,\mathbf{x}_b)$ over all pairs of $(\alpha\ne\beta)$ and input points $\mathbf{x}_a,\mathbf{x}_b$ normalized by the similar average for coinciding models $\alpha=\beta$. The second metric is based on the time derivative of the model $\alpha$ loss at $\mathbf{x}_a$
\begin{equation}\label{eq:EE_nll_loss_derivative}
    \frac{d}{dt} L(\mathbf{f}_\alpha(\mathbf{x}_a),c_a) = - \frac{1}{B} \sum\limits_{\beta=1}^M\sum\limits_{b=1}^B \Theta^{\text{nll}}_{\alpha\beta}(\mathbf{x}_a,\mathbf{x}_b).
\end{equation}
We see that in the $\frac{d}{dt}L_\alpha$ there are contributions from same change of model $\alpha$ outputs based on its own loss $\mathcal{L}_\alpha$, and contributions from changes of model $\alpha$ outputs based on the losses of other models  $\mathcal{L}_\beta, \beta\ne\alpha$. Thus the second contribution represents the effect of model interactions on the training dynamic. Specifically, the metric describes the relative weight of the second contribution as
\begin{equation}\label{eq:NTK_metric_2}
    \left\langle(\partial_t \log p_{c}|_{diff})^2 / (\partial_t \log p_{c}|_{same})^2\right\rangle (m) \equiv \frac{\sum\limits_{\alpha, a} \left(\frac{1}{B}\sum\limits_{b=1}^B\sum\limits_{\beta\in\mathcal{B}_{m,\alpha}} \Theta^{\text{nll}}_{\alpha\beta}(\mathbf{x}_a,\mathbf{x}_b)\right)^2}{\sum\limits_{\alpha, a} \left(\frac{1}{B}\sum\limits_{b=1}^B\Theta^{\text{nll}}_{\alpha\alpha}(\mathbf{x}_a,\mathbf{x}_b)\right)^2}.
\end{equation}
Here $\mathcal{B}_{m\alpha}$ is any set of models $\beta\ne\alpha$ with size $m$. Finally, the third metric is obtained using the formula~\eqref{eq:NTK_metric_2} with $\Theta^{\text{target}}$ instead of $\Theta^{\text{nll}}$. Although it does not have such an explicit interpretation as the second metric, its calculated based only on NTK itself whereas the second metric also uses information about target class probabilities. All the three metrics at initialization for various network widths and for ensemble of size $M=15$ in individual regime (normally distributed modulations $u^l_\alpha$) and collective regime (all modulations distributed as $u^l_\alpha \sim \mathcal{N}(\frac{1}{\sqrt{2}}, \frac{1}{2})$) can be found in Figure~\ref{fig:NTK_stat_at_init}.

\subsection{MC Dropout}
To augment our results with BatchEnsembles, we also performed similar experiments with MC Dropout. More specifically, we  consider well-known Gaussian MC Dropout in which masks come from shifted normal distribution $\mathcal{N}(1, \sigma)$. In addition, we study \textit{centered} Gaussian MC Dropout variation in which masks come form centered Gaussian distribution $\mathcal{N}(0, \sigma)$. Centered Gaussian MC Dropout is close to BatchEnsemble with the only difference in non-learnable modulations. Additionally, we consider two approaches to training for these models. One, where we sample masks independently for each forward pass, as in classical Dropout, and another, where we create a set of masks beforehand, one for each model in the ensemble, and reuse them both at train and test time.

  In Figure~\ref{fig:gaussian_mc_dropout_shifted} we evaluate shifted Gaussian MC Dropout trained with infinite set of masks. For this experiment we set $\mathcal\sigma = 0.1$. One can see that in accordance with our theory it is in \textit{collective} regime, i.e. each model performs well but they are highly correlated. As expected, there is no optimal performance for some specific number of models. Instead, ensemble accuracy asymptotically approaches some limit, as diversity of available models is depleted with increase in their number. While increase in ensemble size causes almost monotonic increase in ensemble accuracy, high correlation of individual models limits this improvement, allowing for only modest gains in ensemble accuracy even with large number of models in ensemble. We can also observe, that individual model accuracy does not suffer from increase in ensemble size. 

  We also consider Centered Gaussian MC Dropout with independent mask sampling during training, see Figure~\ref{fig:gaussian_mc_dropout_no_shift}. As expected the model has hardly learned anything, since  applying  random  zero-centered  mask  at  each  forward  pass makes parameter gradients face random directions at each SGD step and thus there is no direction in the parameter space which will simultaneously reduce the loss of the majority of ensemble models.
 
  Figure~\ref{fig:MC-Dropout-zero-centred} depicts results for Centered Gaussian MC Dropout, but this time with fixed number of masks used both during training and test. In this configuration, zero-centered MC Dropout can learn quite well. It can be seen that it is in \textit{individual} regime, where models are not correlated, but there exists optimal number of models for each width to achieve optimal performance of the ensemble.  

\subsection{LLD ensemble}

Consider a special case of Embedded ensemble when dropout masks $t_{\alpha i}\sim\mathcal{N}(p, 1-p^2)$ are applied only at the last layer. Specifically, a network of (almost) any architecture can be written as
\begin{equation}\label{eq:ref_model_llm}
    f(\mathbf{x}) = \sum\limits_{i=1}^N \varphi_i(\mathbf{W}, \mathbf{x}) \mathbf{c_i}.
\end{equation}
where $\mathbf{c}_i$ are weights of the last linear layer and $\varphi_i(\mathbf{W}, \mathbf{x})$ are (non-linear) feature maps obtained by previous layers of the network with weights $\mathbf{W}$. Then, embedded ensemble with last layer modulations is constructed as 
\begin{equation}
    f_\alpha(\mathbf{x}) = \sum\limits_{i=1}^N \varphi_i(\mathbf{W}, \mathbf{x})u_{\alpha i} \mathbf{c_i}.
\end{equation}
At inference the prediction of such ensemble can be constructed at single forward pass small additional computational time  compared to reference model~\eqref{eq:ref_model_llm}. Specifically, at inference we apply a single averaged modulation $u_i$
\begin{equation}
    f_\text{ens}(\mathbf{x}) = \frac{1}{M}\sum\limits_{\alpha=1}^M f_\alpha(\mathbf{x}) = \sum\limits_{i=1}^N \varphi_i(\mathbf{W}, \mathbf{x})u_{i} \mathbf{c_i}, \quad u_i=\frac{1}{M}\sum\limits_{\alpha=1}^M u_{\alpha i}.
\end{equation}
During training  we again calculate all required gradients in a single backward calculation with a small additional computational time  compared to reference model~\eqref{eq:ref_model_llm}. Specifically, at train time we calculate the predictions of all ensemble models $f_\alpha(\mathbf{x})$ in a single forward pass, and then use it to calculate the train regime ensemble loss $L_{\text{ens, train}}$ to which backward pass will be applied.
\begin{equation}
    L_{\text{ens, train}} = \frac{\gamma(M)}{M} \sum\limits_{\alpha=1}^M \frac{1}{B}\sum\limits_{b=1}^B L(f_\alpha(\mathbf{x}_b), y_b).
\end{equation}
Computational graph of this backward pass is different from the computational graph of the reference model only at the last layer of the network and additional number of operations can be estimated as $O(MNC)$ for a single input $\mathbf{x}$.

We trained ensembles in this configuration with varied number of applied masks, and plot the results on the Figure~\ref{fig:last_layer_dropout}. We observe, that models in the ensemble operate in individual regime, Ensemble has a clear maximum performance at a certain ensemble size $M > 1$. We also observe, that increasing network's width moves the point of optimal performance further to the right. This is further confirmed by the fact that average correlation between individual models is also lower for the wider network. At the same time, relative gain of ensembling over using just one model is lower for the wider networks.

We also take a look at the dependence of LLD ensemble's behavior on the value of the $p$ parameter, which dictates variance and mean of the sampling distribution, from which masks are obtained. We can see, that increasing $p$ leads to flattening of all curves, which is perfectly explainable by thinking about the limit of the distribution as $p$ becomes closer to 1, eventually producing just a vector of ones. This makes all models exactly the same, and ensembling no longer produces any benefits. At the same time, increasing number of masks does not lead to detrimental effect of individual model accuracy, since no matter how many modulations we apply on the last layer, they are all the same and do not depend on the increasing number of masks in any way. We also see that mean model correlation increases with the $p$, which is also consistent with the asymptotic behavior of the models as $p$ approaches 1.

\begin{figure}[tb]
\centering
\includegraphics[scale=0.5]{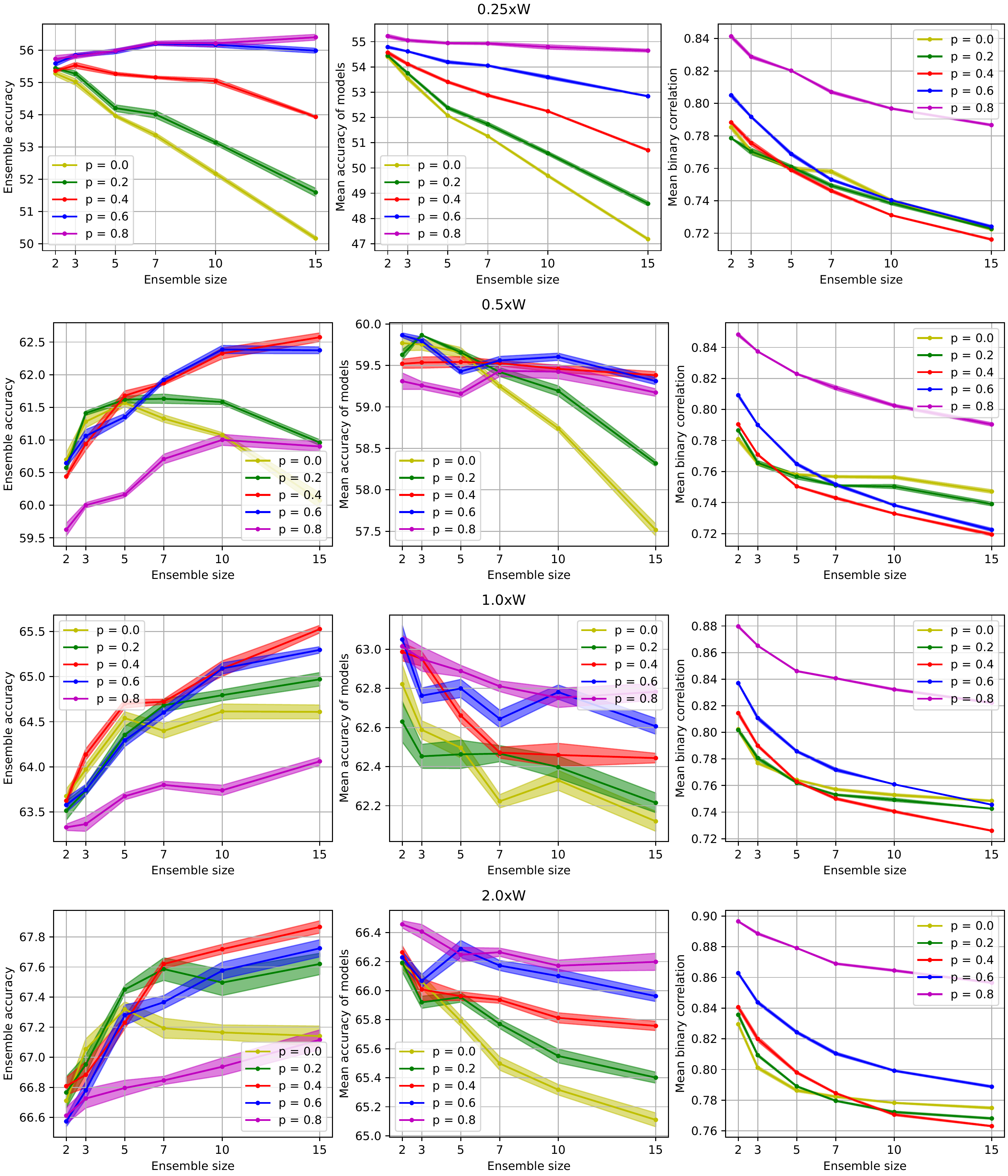}
\caption{Behavior of LLD ensemble with masks of form $\mathcal{N}(p, 1-p^2)$. Different colors represent different values of parameter $p$, while rows belong to different widths of base model.}
\label{fig:last_layer_dropout}
\end{figure}

\end{document}